\documentclass[11pt]{article}
\usepackage[utf8]{inputenc}

\usepackage{mystyle}

\begin{document}

\title {\huge On the Global Convergence of Actor-Critic: A Case for Linear Quadratic Regulator with Ergodic Cost}

\author{Zhuoran Yang\thanks{Department of Operations Research and Financial Engineering, Princeton University}\quad~ Yongxin Chen\thanks{School of Aerospace Engineering, Georgia Institute of Technology}\quad~Mingyi Hong\thanks{Department of Electrical and Computer Engineering, University of Minnesota} \quad~Zhaoran Wang\thanks{Department of Industrial Engineering and Management Sciences, Northwestern  University}}

\date{}

\maketitle


\begin{abstract}
Despite the empirical success of the actor-critic algorithm, its theoretical understanding lags behind. In a broader context, actor-critic can be viewed as an online alternating update algorithm for bilevel optimization, whose convergence is known to be fragile. To understand the instability of actor-critic, we focus on its application to linear quadratic regulators, a simple yet fundamental setting of reinforcement learning. We establish a nonasymptotic convergence analysis of actor-critic in this setting. In particular, we prove that actor-critic finds a globally optimal pair of actor (policy) and critic (action-value function) at a linear rate of convergence. Our analysis may serve as a preliminary step towards a complete theoretical understanding of bilevel optimization with nonconvex subproblems, which is NP-hard in the worst case and is often solved using heuristics. 
\end{abstract}


\section{Introduction}\label{eq:intro}
The actor-critic algorithm \citep{konda2000actor} is one of the most used algorithms in reinforcement learning \citep{mnih2016asynchronous}. Compared with the classical policy gradient algorithm \citep{williams1992simple}, actor-critic tracks the action-value function (critic) in policy gradient in an online manner, and alternatively updates the policy (actor) and the critic. On the one hand, the online update of critic significantly reduces the variance of policy gradient and hence leads to faster convergence. On the other hand, it also introduces algorithmic instability, which is often observed in practice \citep{islam2017reproducibility} and parallels the notoriously unstable training of generative adversarial networks \citep{pfau2016connecting}. Such instability of actor-critic originates from several intertwining challenges, including (i) function approximation of actor and critic, (ii) improper choice of stepsizes, (iii) the noise arising from stochastic approximation, (iv) the asynchrony between actor and critic, and (v) possibly off-policy data used in the update of critic. As a result, the convergence of actor-critic remains much less well understood than that of policy gradient, which itself is open. Consequently, the practical use of actor-critic often lacks theoretical guidance. 

In this paper, we aim to theoretically understand the algorithmic instability of actor-critic. In particular, under a bilevel optimization framework, we establish the global rate of convergence and sample complexity of actor-critic for linear quadratic regulators (LQR) with ergodic cost, a simple yet fundamental setting of reinforcement learning \citep{recht2018tour}, which captures all the above challenges. Compared with the classical two-timescale analysis of actor-critic \citep{borkar1997actor}, which is asymptotic in nature and requires finite action space, our analysis is fully nonasymptotic and allows for continuous action space. Moreover, beyond the convergence to a stable equilibrium obtained by the classical two-timescale analysis, we for the first time establish the linear rate of convergence to a globally optimal pair of actor and critic. In addition, we characterize the required sample complexity. As a technical ingredient and byproduct, we for the first time establish the sublinear rate of convergence for the gradient temporal difference algorithm \citep{sutton2009fast,sutton2009convergent} for ergodic cost and dependent data, which is of independent interest.  

Our work adds to two lines of works in machine learning, stochastic analysis, and optimization: 
\vskip4pt
\noindent (i) Actor-critic falls into the more general paradigm of bilevel optimization \citep{luo1996mathematical, dempe2002foundations, bard2013practical}. Bilevel optimization is defined by two nested optimization problems, where the upper-level optimization problem relies on the output of the lower-level one. As a special case of bilevel optimization, minimax optimization is prevalent in machine learning. Recent instances include training generative adversarial neural networks \citep{goodfellow2014generative}, (distributionally) robust learning \citep{sinha2017certifiable}, and imitation learning \citep{ho2016generative, cai2019global}. Such instances of minimax optimization remain challenging as they lack convexity-concavity in general \citep{du2018linear, sanjabi2018solving, chen2018training, rafi2018noncon, lin2018solving, dai2017learning, dai2018boosting, dai2018sbeed, lu2019understand}. The more general paradigm of bilevel optimization remains even more challenging, as there does not exist a unified objective function for simultaneous minimization and maximization. In particular, actor-critic couples the nonconvex optimization of actor (policy gradient) as its upper level and the convex-concave minimax optimization of critic (gradient temporal difference) as its lower level, each of which is challenging to analyze by itself. Most existing convergence analysis of bilevel optimization is based on two-timescale analysis \citep{borkar1997stochastic}. However, as two-timescale analysis abstracts away most technicalities via the lens of ordinary differential equations, which is asymptotic in nature, it often lacks the resolution to capture the nonasymptotic rate of convergence and sample complexity, which are obtained via our analysis. 
\vskip4pt
\noindent (ii) As a proxy for analyzing more general reinforcement learning settings, LQR is studied in a recent line of works \citep{bradtke1993reinforcement, recht2018tour, fazel2018global, tu2017least, tu2018gap, dean2018regret, dean2018safely, simchowitz2018learning, dean2017sample, hardt2018gradient}. In particular, a part of our analysis is based on the breakthrough of \cite{fazel2018global}, which gives the global convergence of the population-version policy gradient algorithm for LQR and its finite-sample version based on the zeroth-order estimation of policy gradient based on the cumulative reward or cost. However, such zeroth-order estimation of policy gradient often suffers from large variance, as it involves the randomness of an entire trajectory. In contrast, actor-critic updates critic in an online manner, which reduces such variance but also introduces instability and complicates the convergence analysis. In particular, as the update of critic interleaves with the update of actor, the policy gradient for the update of actor is biased due to the inexactness of critic. Meanwhile, the update of critic has a ``moving target'', as it attempts to evaluate an actor that evolves along the iterations. A key to our analysis is to handle such asynchrony between actor and critic, which is a ubiquitous challenge in bilevel optimization. We hope our analysis may serve as the first step towards analyzing actor-critic in more general reinforcement learning settings.



 \vspace{5pt}
\noindent{\bf Notation.} For any integer $n>0$, we denote $\{ 1, \ldots, n \}$ to be $[n]$. For any symmetric matrix $X$, let  $\svec(X)$ denote the vectorization of  the upper triangular submatrix of $X$ with off-diagonal entries weighted by $\sqrt{2}$. Meanwhile, let $\smat(\cdot)$ be the inverse operation of $ \svec(\cdot)$, which maps a vector to a symmetric matrix. Besides, we denote by $A \otimes _{s} B$ the symmetric Kronecker product of $A$ and $B$. We use $\| v\|_2$ to denote the $\ell_2$-norm of a vector $v$. Finally, for  a matrix $A$, we use $\| A \|$ and $\| A \|_{\fro}$ to denote  its the operator norm and Frobenius norm, respectively.


\section{Background}

In the following, we introduce the background of actor-critic and LQR. In particular, we show that actor-critic can be cast as a first-order online alternating update algorithm for a bilevel optimization problem \citep{luo1996mathematical, dempe2002foundations,bard2013practical}.

\subsection{Actor-Critic Algorithm} \label{sec:rl}

 We consider a Markov decision process, which is defined by $(  \cX, \cU, P,  c , D_0)$. Here $\cX$ and $\cU$ are the   state  and      action spaces, respectively,  $P\colon \cX \times \cU  \rightarrow \cP(\cX)$ is the Markov transition kernel,  $c\colon \cX \times \cU \rightarrow \RR$ is the cost function, and  $D_0 \in \cP(\cX)$ is the distribution of the initial state $x_0 $. For any $t\geq 0$, at the $t$-th time step, the agent takes action $u_t \in \cU$ at state $x_t \in \cX$,
  which incurs a cost $c(x_t, u_t)$ and moves the environment  into a new state  $x_{t+1} \sim P(\cdot \given x_t, u_t)$. A policy  specifies how the action $u_t$ is taken at a given state $x_t$. Specifically,  in order to handle infinite state and action spaces, we focus on a parametrized policy class $ \{ \pi_{\omega} \colon \cX \rightarrow\cP(  \cU) ,   \omega \in \Omega \} $, where $ \omega$ is the parameter of policy $\pi_{\omega}$,  and the agent takes action $u \sim  \pi_{\omega}(\cdot \given x)$ at a given state $x \in \cX$. 
 The agent aims to find a policy that minimizes the infinite-horizon time-average cost, that is, 
  \# \label{eq:cost}
 \minimize _{ \omega \in \Omega}~ J( \omega) = \lim_{T\rightarrow \infty}     \EE    \biggl [ \frac{1}{T}    \sum_{t =  0}^T    c(x_t, u_t)  \bigggiven x_0 \sim D_0, u_t  \sim  \pi_{ \omega }   ( \cdot \given  x_t), \forall t\geq 0\biggr ]. 
  \#
 Moreover, for policy $\pi_{ \omega}$, we define the (advantage) action-value and state-value functions respectively as 
 \#\label{eq:value_function}
 Q_{ \omega} (x, u ) =   \sum_{t\geq 0}\EE_{ \omega} \bigl[ c(x_t, u_t) \given x_0 = x, u_0 = u \bigr] - J( \omega ), \qquad  V_{ \omega} (x ) =  \EE_{u\sim \pi_{\omega} (\cdot \given x) } \bigl[ Q_{ \omega} ( x, u)\bigr],  
 \#
 where we use $\EE_{\omega}$ to indicate that the state-action pairs $\{ (x_t, u_t)\}_{t\geq 1}$ are obtained from policy $\pi_{ \omega}$.
  
 Actor-critic is based on the idea of solving the minimization problem in \eqref{eq:cost} via first-order optimization, which uses an estimator of $\nabla _{\omega}J(\omega)$. In detail, by the policy gradient theorem \citep{sutton2000policy, baxter2001infinite, konda2000actor}, we have
\#\label{eq:pg_thm}
\nabla_{\omega} J(\omega) = \EE_{x \sim \rho_\omega, u \sim \pi_\omega(\cdot \given x)}  \big [ \nabla _{\omega} \log \pi_\omega (u \given x) \cdot Q_{\omega} (x,u)\big ], 
\#
where $ \rho_\omega \in \cP(\cX)$ is the stationary distribution induced by $\pi_\omega$. 
Based on \eqref{eq:pg_thm}, actor-critic \citep{konda2000actor} consists two steps: (i) a policy evaluation step that estimates the action-value function $Q_{\omega}$ (critic) via temporal difference learning \citep{dann2014policy}, where $Q_{\omega}$ is estimated using a parametrized function class $\{Q^{\theta} \colon \theta \in \Theta \}$,  and (ii) a policy improvement step that updates the parameter $\omega$ of policy $\pi_\omega$ (actor) using a stochastic version of the policy gradient in \eqref{eq:pg_thm}, where $Q_{\omega}$ is replaced by the corresponding estimator $Q^\theta$. 
 
  As shown in 
 \cite{yang2018convergent}, actor-critic can be cast as solving a bilevel optimization problem, which takes the form
 \#
  \minimize_{\omega \in \Omega} \quad & \EE_{x\sim \rho_{\omega}, u \sim \pi_{\omega}(\cdot \given x)  }  \bigl [ Q^\theta(x, u) \bigr ], \label{eq:upperbilevel} \\
 \text{subject to}\quad  &( \theta, J) = \argmin_{\theta \in \Theta, J \in \RR}  \EE_{x\sim \rho_{\omega}, u\sim \pi_{\omega}(\cdot \given s)  } \Bigl \{ \bigl [  Q^\theta(x,u)  + J - c(x, u) - ( \cB^\omega Q ^{\theta} )(x, u) \bigr]^2 \Bigr \},\label{eq:lowerbilevel}
 \#
 where $\cB^\omega$ is an operator that depends on $\pi_{\omega}$.  In this problem, the actor and   critic  correspond to  the upper-level and lower-level variables, respectively.  Under this framework, the policy update can be viewed as a stochastic gradient step for the upper-level problem in \eqref{eq:upperbilevel}. 
 The objective in  \eqref{eq:lowerbilevel}  is usually the mean-squared Bellman error or mean-squared projected Bellman error \citep{dann2014policy}.
 Moreover, when $\cB^{\omega}$ is the Bellman evaluation operator associated with $\pi_{\omega}$ and we solve the lower-level problem in \eqref{eq:lowerbilevel} via stochastic  semi-gradient descent, we obtain the TD(0) update  for policy evaluation \citep{sutton1988learning}.  
 Similarly, when  $\cB^{\omega} $ is the projected Bellman evaluation operator  associated with  $\pi_{\omega}$, solving the lower level problem naturally 
 recovers the GTD2 and TDC algorithms for policy evaluation \citep{bhatnagar2009convergent}. 
  Therefore, the actor-critic algorithm is a first-order online algorithm for the bilevel optimization problem in \eqref{eq:upperbilevel} and \eqref{eq:lowerbilevel}. We remark  that bilevel optimization contains a family of extremely challenging  problems. Even when the objective functions are linear, bilevel programming is NP-hard \citep{hansen1992new}. In practice, various heuristic algorithms are applied  to solve them approximately \citep{sinha2018review}.

 \subsection{Linear Quadratic Regulator}\label{sec:lqr}
 
As the simplest optimal control problem,   linear quadratic regulator  serves as a perfect baseline to examine the performance of reinforcement learning methods.   Viewing LQR from the lens of MDP,  the state and action spaces are  $\cX=\real^{d}$ and $\cU=\real^{k}$, respectively. Besides,  the  state transition dynamics and cost function are specified by
\#\label{eq:lqr_model}
x_{t+1} = A x_t + B u_t + \epsilon_t,   \qquad c(x, u )=x ^\top Q x  + u ^\top R u, 
\#
where $\epsilon_t \sim N(0, \Psi)$ is the random noise that is  i.i.d. for each $t\geq 0$, and $A$, $B$, $Q$, $R$, $\Psi$ are matrices of proper dimensions with 
 $Q, R, \Psi \succ 0$. 
 Moreover, we assume that the dimensions $d$ and $k$ are fixed throughout this paper.
For the problem of  minimizing the infinite-horizon time-average cost $\limsup_{T\rightarrow \infty}    T^{-1} \sum_{t=0}^{T} \EE [ c(x_t, u_t)]$ with $x_0 \sim D_0$, it is known that the optimal action  are linear in the corresponding state \citep{zhou1996robust, anderson2007optimal, bertsekas2012dynamic}. Specifically, the optimal actions  
$\{ u_t^*\}_{t\geq 0}$ satisfy $u_t^* =- K ^*x_t$ for all $t\geq 0$, where  $K^* \in \RR^{k \times d}$ 
can be written as 
$K^*  = ( R + B^\top P^* B)^{-1}  B^\top P^* A $, with  $P^*$ being the solution to the discrete algebraic Riccati equation 
\#\label{eq:riccati}
P^* = Q + A ^\top P^* A + A^\top P^* B  ( R + B^\top P^* B)^{-1}  B^\top P^* A.
\#

In the optimal control literature, it is  common to  solve  LQR by    first estimating matrices $A$, $B$, $Q$, $R$ and then solving the Riccati equation in \eqref{eq:riccati}
with these matrices replaced by their estimates. Such an approach is known as  model-based as  it requires estimating the   model parameters and the performance of the planning step in   \eqref{eq:riccati} hinges on how well the true model is estimated. See, e.g,  \cite{dean2017sample, tu2018gap} for  theoretical guarantees of model-based methods.

In contrast,   from a purely data-driven perspective, the framework of model-free reinforcement learning   offers a general treatment for optimal control problems without the prior knowledge of the model.  Thanks to its  simple structure,  LQR enables us to assess the  performances of reinforcement learning algorithms from a theoretical perspective. Specifically, it is shown that policy iteration  \citep{ bradtke1993reinforcement,bradtke1994adaptive, meyn1997policy}, adaptive dynamic programming \citep{powell2011review}, and policy gradient methods  \citep{fazel2018global, malik2018derivative, tu2018gap} are all able to obtain the optimal policy of LQR. 
Also see \cite{recht2018tour} for a thorough review of reinforcement learning methods in the setting of LQR.


\section{Actor-Critic Algorithm for LQR}

In this section, we establish the actor-critic algorithm for the LQR problem introduced in \S\ref{sec:lqr}. Recall that the optimal policy of LQR is a linear function of the state. 
Throughout the rest of this paper,  we   focus on the family of linear-Gaussian policies 
\#\label{eq:gaussian_policy}
  \bigl \{  \pi_K (\cdot \given x) =  N( - Kx , \sigma^2   I_d), K \in \RR^{k \times d} \bigr \},
\#
where $\sigma > 0$ is a fixed constant. 
That is, for any $t\geq 0$, at state $x_t$, we could write the action $u_t$ by $  u_t = - Kx_t + \sigma \cdot \eta_t$, where $\eta_t \sim N(0, I_k)$.  We note that if $\sigma = 0$, then the optimal policy $ u = - K^* x$ belongs to  our policy class.
 Here, instead of focusing on deterministic policies,  we adopt  Gaussian policies to encourage exploration. 
 For policy $\pi_K$, the corresponding time-average cost $J(K)$,  state-value function  $V_K$, and action-value function $Q_K$ are specified  as in \eqref{eq:cost} and \eqref{eq:value_function}, respectively.
 
In the following, we first establish the policy gradient  and value functions for the ergodic  LQR in \S\ref{sec:pg}. Then, in \S\ref{sec:ac},  we present the on-policy  natural actor-critic algorithm, which is further extended to the off-policy setting in \S\ref{sec:offac}.

\subsection{Policy Gradient Theorem for Ergodic LQR}\label{sec:pg}
For any policy $\pi_K$, 
by \eqref{eq:lqr_model}, 
the state dynamics is given by a linear dynamical system  
\#\label{eq:new_dyn}
x_{t+1} = (A - BK ) x_t + \varepsilon_t , \qquad \text{where}\quad \varepsilon _t = \epsilon_t + \sigma\cdot  B\eta_t \sim N(0, \Psi _{\sigma}  ).
\#
Here we define $ \Psi_{\sigma}:=\Psi + \sigma^2 \cdot B B^\top$   in \eqref{eq:new_dyn}
 to simplify the notation.  
 It is known that, when $\rho(A - BK) < 1$, the Markov chain  in \eqref{eq:new_dyn} has stationary distribution $N(0, \Sigma_K)$, denoted by $\rho_K$ hereafter, where $\Sigma_K$  is the unique positive definite  solution to the  Lyapunov equation
 \#\label{eq:cov_equ}
 \Sigma_K = \Psi_\sigma + (  A - BK)  \Sigma_K ( A- BK)^\top.
 \#
In the following proposition, we establish $J(K)$, the value functions, and the gradient  $\nabla_K J(K)$.

\begin{proposition} \label{prop:pg}
	For any $K \in \RR^{k\times d}$ such that $\rho(A - BK )< 1$, let $P_K$ be  the unique positive definite solution to the Bellman equation
	\#\label{eq:bellman}
	P_K   =   (Q + K^\top R K )+  ( A- BK) ^\top  P_K ( A- BK)  . 
	\#
	In the setting of LQR, for policy $\pi_K$, 	
	 both the state- and action-value functions are     quadratic. Specifically, we have 
	\#
	V_K(x) & = x^\top P_K x - \tr(P_K \Sigma_K) , \label{eq:vk} \\
	Q_K(x,u) & =  x ^\top \Theta_K^{11} x + x^\top \Theta_K ^{12} u + u^\top \Theta_{K}^{21} x + u^\top \Theta_{K}^{22} u -  \sigma ^2 \cdot \tr (R + P_K BB^\top)- \tr(P_K \Sigma_K)  , \label{eq:qk}
	\# 
	where $\Sigma_K$ is specified  in \eqref{eq:cov_equ}, and we define matrix $\Theta_K $ by 
	\#\label{eq:ThetaK}
	\Theta_K = \begin{pmatrix}
	\Theta_K^{11} & \Theta_{K}^{12}\\
		\Theta_{K}^{21}  & \Theta_K^{22} 
	\end{pmatrix}  = 
	\begin{pmatrix}
		Q + A^\top P_K A  & A ^\top P_K B \\
		B^\top P_K A   & R + B^\top P_K B 
	\end{pmatrix} .
	\#
 Moreover, the time-average cost $J(K)$ and its gradient are given by 
	\#
	J(K) &  =  \tr \bigl [  ( Q+ K^\top R K ) \Sigma_K \bigr ] + \sigma^2 \cdot \tr(R) = \tr ( P_K \Psi_{\sigma}) +  \sigma ^2 \cdot \tr (R).\label{eq:cost_K2} \\
  \nabla_K J(K)  & = 2 \bigl [  ( R   +  B^\top P_K B ) K - B^\top P_K A \bigr ] \Sigma _K = 2 E_K \Sigma _K, \label{eq:grad_cost} 
	\# 
	where we define $E_K :=   ( R   +  B^\top P_K B ) K - B^\top P_K A    $.
\end{proposition}

\begin{proof}
	See \S \ref{proof:prop:pg} for a detailed proof.  
	\end{proof}
 
To see the connection between \eqref{eq:grad_cost} and the policy gradient theorem in \eqref{eq:pg_thm}, note that by direct computation we have 
\#\label{eq:score_function}
\nabla_K \log \pi_K ( u \given x) = \nabla_K \bigl [  - (2 \sigma^2 )^{-1} \cdot  (u + K x )^2 \bigr ] = - \sigma^{-2} \cdot (u+K x) x^\top.
\#
Thus, combining,  \eqref{eq:qk}, \eqref{eq:score_function},  and the fact that $u = -Kx + \sigma \cdot \eta $ under $\pi_K$, the right-hand side of \eqref{eq:pg_thm} can be written as 
\$
& \EE_{x\sim \rho_K, u \sim \pi_K} \bigl [ \nabla_K \log \pi_K (u\given x) \cdot Q_K(x, u )  \bigr ] = - \sigma^{-2} \cdot \EE_{x\sim \rho_K, \eta \sim N(0, I_k) }  [ \sigma \cdot \eta x^\top \cdot Q_{K} (x, -Kx+\sigma \cdot\eta  )] . 
\$
Recall that for $\eta \in N(0, I_k)$, Stein's identity \citep{stein1981estimation}, $\EE [ \eta \cdot  f(\eta ) ] = \EE [ \nabla f(\eta)]$, holds for all differentiable function $f \colon \RR^k \rightarrow \RR$, which implies that  
\# 
 \nabla_K J(K) &  =  - \EE_{x\sim \rho_K, \eta \sim N(0, I_d)} [( \nabla _u Q_{K}) (x, -K x+\sigma \cdot \eta )  \cdot x^\top ]  \notag\\
 &   = -  2  \EE_{x\sim \rho_K, \eta \sim N(0, I_d)} \{  [ ( R + B^\top P_K B) ( - Kx + \sigma\cdot \eta)  + B^\top P_K A x ] x^\top \} \notag \\
  & = 2 \bigl [  ( R   +  B^\top P_K B ) K - B^\top P_K A \bigr ] \Sigma _K= 2 E_K \Sigma_K .\label{eq:verify_pg} 
\#
Thus, \eqref{eq:grad_cost} is exactly the policy gradient theorem   \eqref{eq:pg_thm} in the setting of LQR.
Moreover, it is worth  noting that Proposition \ref{prop:pg} also holds for $\sigma = 0$. Thus, setting $\sigma = 0$ in \eqref{eq:verify_pg}, we obtain 
\$
\nabla_K J(K) = - \EE_{x\sim \rho_K } [ ( \nabla _u Q_{K}) (x, -K x  )  \cdot x^\top ] = \EE_{x\sim \rho_K} \bigl [ (\nabla _{u} Q_{K})(x, u) \big \vert_{u = \pi_K(x)} \nabla_{K} \pi_K(x)  \bigr],  
\$ 
where $\pi_K(x) = - K x$. Thus we obtain   the 
 deterministic  
 policy gradient theorem  \citep{silver2014deterministic} for LQR. Although the optimal policy for LQR is deterministic, due to the lack of exploration, behaving according to a deterministic policy
   may lead to suboptimal solutions. Thus, we focus on the family of stochastic policies where   a  Gaussian noise $\sigma \cdot \eta$ is added to the action so as to promote exploration.

\subsection{Natural Actor-Critic Algorithm} \label{sec:ac}
Natural policy gradient updates the variable along the steepest direction with respect to Fisher metric.
For the Gaussian policies defined in \eqref{eq:gaussian_policy}, by \eqref{eq:score_function},   the Fisher's information of policy $\pi_K$, denoted by $\cI(K)$,  is given by 
\#\label{eq:fisher_info}
 [ \cI ( K) ]_{(i,j), (i',j')}  & =  \EE_{x\sim \rho_K, a \sim \pi_K } \bigl [ \nabla_{K_{ij}}  \log \pi_K ( u \given x) \cdot \nabla_{K_{i'j'}}  \log \pi_K ( u \given x) \bigr ]  \notag \\
&    = \sigma^{-2} \cdot \EE_{x\sim \rho_K, \eta \sim N(0, I_k) }  [  \eta_{i}  x_{j} \cdot \eta_{i'} x_{j'}   ] = \sigma^{-2} \cdot \ind\{ i = i'\}  \cdot [\Sigma_K ]_{jj'},
\#
where $i,i' \in [k]$, $j, j' \in [d]$, $K_{ij}$ and $K_{i'j'}$ are the $(i,j)$- and $(i',j')$-th entries of $K$, respectively. Thus, in view of \eqref{eq:grad_cost} in Proposition \ref{prop:pg} and \eqref{eq:fisher_info}, 
natural policy gradient algorithm updates the policy parameter  in the direction of  
\$
[ \cI(K)]^{-1} \nabla_K J(K) = \nabla_K J(K)\Sigma_K^{-1}  = E_K.
\$
  By \eqref{eq:ThetaK}, we can write $E_K $ as $\Theta_K^{22} K - \Theta_{K}^{21}$, where $\Theta_K$ is the coefficient matrix of the quadratic component of $Q_K$.
Such a connection lays the foundation of the natural actor-critic algorithm. Specifically, in each iteration of the   algorithm, the actor updates the policy via $K - \gamma \cdot ( \hat \Theta ^{22} K - \hat \Theta^{21})$, where $\gamma$ is the stepsize and $\hat \Theta$ is an estimator of $\Theta_K$ returned by any policy evaluation algorithm. 
We present such a  general natural actor-critic method in Algorithm~\ref{algo:ac},  under the assumption that we are given a stable  policy $K_0$ for initialization. Such an assumption is standard in  literatures on  model-free methods for LQR \citep{dean2018regret, fazel2018global, malik2018derivative}. 

\begin{algorithm} [htbp]
	\caption{Natural Actor-Critic Algorithm for Linear Quadratic Regulator} 
	\label{algo:ac} 
	\begin{algorithmic} 
		\STATE{{\textbf{Input:}} Initial policy  $\pi_{K_0}$ such that $\rho(A - BK_0) < 1$,   stepsizes $\gamma  $ for policy update, and a policy evaluation algorithm.}  
	 	\STATE{{\textbf{Initialization:}} Set the current policy $\pi_K$ by letting $K \leftarrow K_0$.}
		\WHILE{updating current policy}
		\STATE{{\textbf{Critic step.}} Estimate $\Theta_K$ in \eqref{eq:ThetaK} via a  policy evaluation algorithm, e.g., the on-policy GTD algorithm  (Algorithm \ref{algo:gtd}), which returns an estimator $\hat \Theta$ of $\Theta_K$. 
		}
		\STATE{{\textbf{Actor step.}}  Update the  policy parameter by  
		$ K \leftarrow  K - \gamma \cdot    ( \hat \Theta^{22} K - \hat \Theta^{21}  ) $.}
		\ENDWHILE
		\STATE{{\textbf{Output:}} The final policy $\pi_K$,  matrix   $\hat \Theta$ that estimates $\Theta_K$,  and $\hat J$ that approximates $J(K)$.}
	\end{algorithmic}
\end{algorithm}

To obtain an online actor-critic algorithm, in the sequel, we propose an online policy evaluation algorithm 
 for ergodic LQR based on temporal difference learning.  Let $\pi_K$ be the policy of interest. 
 For  notational simplicity, 
for any state-action pair $(x,u )\in \RR^{d+k}$, we define the feature function 
\#\label{eq:feature_map}
\phi(x, u) = \svec\Biggl [ \begin{pmatrix}
	x \\ u 
\end{pmatrix} \begin{pmatrix}
	x \\ u 
\end{pmatrix}^\top \Biggr ],
\#
and  denote by $\svec(\Theta_K)$ by $\theta_K^*$. 
Using this notation, the quadratic component in $Q_K$ can be written as $\phi(x,u) ^\top \theta_K^*$, and the Bellman equation for $Q_K$ becomes 
\#\label{eq:new_bellmanQ}
 \la \phi(x, u ), \theta_K ^*\ra =  c(x,u)  - J(K) + \bigl \la \EE [ \phi(x', u' ) \given x, u], \theta_K ^* \bigr  \ra , \; \forall (x, u ) \in \RR^{d+k}.
\#
In order to further simplify the notation,   hereafter,   we define   $\vartheta _K^* = ( J (K) , {\theta_K^*} ^\top ) ^\top$, denote by $\EE_{(x,u)}$ the expectation with respect to $x\sim \rho_K$ and $u \sim \pi_K(\cdot \given x)$, and let $(x', u')$ be the state-action pair subsequent to $(x,u)$.

  Furthermore, to estimate $J(K)$ and $\theta_K^*$ in \eqref{eq:new_bellmanQ}  simultaneously,  we define 
\#\label{eq:gtd_mat}
\Xi_K   = \EE_{(x, u)  }  \bigl \{  \phi (x, u )  \bigl [ \phi(x, u ) - \phi (x', u') \bigr ]^{\top } \bigr \}, \qquad  b_K = \EE_{(x ,u)  } \bigl [    c(x,u)    
\phi (x, u) \bigr ],  
\#
 Notice that  $J(K) = \EE_{(x, u) } [ c(x,u)]$. By direct computation, it can be shown that $\vartheta_K^*$ satisfies the following linear equation 
 \#\label{eq:large_le}
 \begin{pmatrix} 
 1 & 0\\
\EE_{(x,u)} [ \phi(x,u)]  & \Xi_K 
 \end{pmatrix}   \begin{pmatrix}
 \vartheta^1\\
 \vartheta^2
 \end{pmatrix}  
 = 
 \begin{pmatrix}
J(K) \\
b_K
 \end{pmatrix},   \#
  whose solution is unique  if and only if $\Xi_K$ in \eqref{eq:gtd_mat} is invertible. The following lemma shows that, when  $\pi_K$ is a stable policy,  $\Xi_K$ is indeed invertible.
 
\begin{lemma} \label{lemma:pe_mat1}
	When $\pi_K$ is stable in the sense that 
	 $\rho(A- BK)< 1$,   $\Xi_K$  defined in \eqref{eq:gtd_mat} is invertible and thus $\vartheta _K^* $ is the unique solution to the linear equation \eqref{eq:large_le}. Furthermore,    that the minimum singular value of the matrix in the left-hand side of \eqref{eq:large_le} is lower bounded by a constant $\kappa_K^*>0$, where $\kappa_K^*$ only depends on $\rho(A- BK)$, $\sigma$, and  $\sigma_{\min} (\Psi)$.   
	  \end{lemma}
 \begin{proof}
 	See \S\ref{sec:proof_lemma_pe_mat1} for a detailed proof.
 	\end{proof}

By this lemma, when $\rho(A - BK ) <1$, policy evaluation for $\pi_K$ can be reduced to finding the unique solution to a  linear equation. Instead of solving the equation directly, it is equivalent to minimize the least-squares loss:
\#\label{eq:least_squares}
\minimize_{\vartheta} \Bigl \{  [ \vartheta^1 - J(K) ] ^2 +  \bigl \| \vartheta^1 \cdot \EE_{(x,u)} [ \phi(x,u) ] + \Xi_K \vartheta^2 -b_K   \bigr \|_2^2 \Bigr \},
\#
where $\vartheta^1 \in \RR$ and $\vartheta^2$ has the same shape as $\svec(\Theta_K)$, which are the two components of $\vartheta$. It is clear that  
 the  global minimizer of \eqref{eq:least_squares} is $\vartheta_K^*$.  Note that   we have  Fenchel's duality $ x^2 =  \sup_{y } \{ 2 x \cdot y -   \cdot y^2 \} $. 
By this relation, we further write \eqref{eq:least_squares} as a minimax optimization problem 
\#\label{eq:new_opt}
\min_{\vartheta \in \cX_{\Theta} }\max _{\omega\in \cX_{\Omega} }  F(\vartheta, \omega) & =   [ \vartheta^1 - J(K) ]  \cdot \omega ^1  \notag \\
&\qquad \qquad  + \bigl \la  \vartheta^1 \cdot \EE_{(x,u)} [ \phi(x,u) ] + \Xi_K \vartheta^2 -b_K , \omega^2 \bigr \ra   - 1/2 \cdot \| \omega\|_2^2   ,
\#
where the dual variable $\omega = (\omega^1, \omega^2)$ has the same shape as $\vartheta$. Here we restrict the primal and dual variables to compact sets  $\cX_{\Theta}$ and $\cX_{\Omega}$  for algorithmic stability, which will be specified in the next section.  Note that the objective in \eqref{eq:new_opt} can be estimated unbiasedly using two consecutive state-action pairs $(x,u)$ and $(x', u')$.  Solving the minimax optimization in \eqref{eq:new_opt} using stochastic gradient method, we obtain the gradient-based temporal difference (GTD) algorithm  for policy evaluation \citep{sutton2009fast, sutton2009convergent}. See Algorithm \ref{algo:gtd} for details. 
 More specifically, by direct computation, we have 
\#
 \nabla_{\vartheta^1} F(\theta, \omega) &=  \omega^1 + \bigl \la \EE_{(x,u)} (\phi  ), \omega^2 \bigr \ra, \qquad \nabla_{\vartheta^2} F(\theta, \omega)  =  \EE_{(x,u)}  \bigl [ (\phi - \phi'   )  \cdot \phi ^\top \omega^2 \bigr ], \label{eq:grad_theta} \\
 \nabla_{\omega^1} F(\theta, \omega) &=    \vartheta^1 - J(K) - \omega^1 , \qquad ~~~~~ \nabla_{\omega^2}F(\theta, \omega)  = \theta^1 \cdot \EE_{(x,u)} (\phi) + \Xi_K \vartheta^{2} - b_K - \omega^2, \label{eq:grad_omega}
\#
where we denote $\phi(x,u)$ and $\phi(x', u')$ by $\phi$ and $\phi'$, respectively. In the GTD algorithm, we update $\vartheta$ and $\omega$ in gradient directions where  the gradients in \eqref{eq:grad_theta} and \eqref{eq:grad_omega} are replaced by their sample estimates. After $T$ iterations of the algorithm,  we output the averaged update 
$
\hat \vartheta  ^2  =   (  \sum_{t=1}^{T} \alpha_t \cdot \vartheta_t^2  )  /  ( \sum_{t=1}^{T} \alpha _t )
$
and use  $\hat \Theta = \smat(\hat \vartheta^2)$ to estimate  $\Theta_K$  in \eqref{eq:ThetaK}, which  is further 
 used in Algorithm \ref{algo:ac} to update the current policy $\pi_K$.   Therefore, we obtain the  online natural actor-critic algorithm \citep{bhatnagar2009natural} for ergodic LQR.
 
Meanwhile, using the perspective of bilevel optimization, similar to \eqref{eq:upperbilevel} and \eqref{eq:lowerbilevel}, our actor-critic algorithm can be viewed as a first-order online algorithm for 
 \#\label{eq:nac_bilevel}
   \minimize_{K \in \RR^{k\times d} }   \EE_{x\sim \rho_{K} , u \sim \pi_{K}  }  \bigl [  \la \phi(x,u) , \vartheta^2  \ra  \bigr ],   \qquad  \text{subject to}\quad   ( \vartheta ,  \omega ) = \argmin_{\theta \in \cX_{\Theta}} \argmax_{  \omega\in \cX_{\Omega} }   F(\vartheta, \omega) ,  
   \#
 where $F(\vartheta, \omega) $ is  defined in \eqref{eq:new_opt} and depends on $\pi_K$.  
In our algorithm, we solve the upper-level problem via natural gradient descent and solve the lower-level saddle point optimization problem using stochastic gradient updates.
 
 Furthermore, we emphasize that our method defined by Algorithms \ref{algo:ac} and \ref{algo:gtd} is online in the sense that each update only requires a single transition. 
 More specifically, let $\{ (x_n, u_n, c_n) \}_{n \geq 0}$ be the sequence of transitions experienced by the agent.  Combining Algorithms \ref{algo:ac} and \ref{algo:gtd} and neglecting the projections, we can write the updating rule as 
 \#\label{eq:combined_update}
 \begin{split}
 K_{n+1}  = K_n - \overline \gamma_n \cdot & \bigl \{ [ \smat(\vartheta_n^2 ) ]^{22}   K_n -[ \smat(\vartheta_n ^2 ) ]^{22}  \bigr \},    \\
 \vartheta_{n+1}  =   \vartheta_n - \overbar  \alpha _n \cdot g_{\vartheta} ( x_n, u_n, c_n, x_{n+1} , u_{n+1})  , & \quad 
 \omega_{n+1}    = \omega_n +  \overbar  \alpha _n \cdot g_{\omega} ( x_n, u_n, c_n, x_{n+1} , u_{n+1})  ,
 \end{split} 
 \#
 where $g_{\vartheta}$ and $g_{\omega}$ are the upda{te directions of $\vartheta$ and $\omega$ whose definitions are clear from Algorithm \ref{algo:gtd},  $\{\overline \gamma_n\}$  and $\{ \overline \alpha_n \}$ are the stepsizes. Moreover, there exists a monotone increasing  sequence    $\{N_t \}_{t\geq 0}$ such that 
 $\overline \gamma_n = \gamma $ if $n = N_t$ for some $t$ and $\overline \gamma_n = 0$ otherwise. 
 Such a choice of the stepsizes reflects the intuition that, although  both the actor and the critic are updated simultaneously,   the critic should be updated in a faster pace. From the same viewpoint, classical actor-critic algorithms \citep{konda2000actor, bhatnagar2009natural, grondman2012survey} establish convergence results under the assumption that 
  \$
 \sum_{n \geq 0} \overline \gamma_n = \sum_{n \geq 0} \overline \alpha_n = \infty, \qquad \sum_{n \geq 0} (\overline \gamma_n ^2 + \overline \alpha_n^2 ) < \infty, \qquad \lim_{n\rightarrow \infty}  \overline \gamma_n / \overline \alpha_n = 0.
 \$
The condition that $\overline \gamma_n / \overline \alpha_n = 0$ ensures that the critic updates in a faster timescale, which enables the  asymptotic analysis  utilizing two-timescale stochastic approximation \citep{borkar1997stochastic,kushner2003stochastic}.   However, such an approach uses two ordinary differential   equations (ODE) to approximate the updates in \eqref{eq:combined_update} and thus only offers asymptotic convergence results. In contrast, as shown in \S\ref{sec:theory},  our choice of the stepsizes yields nonasymptotic  convergence results which shows that natural actor gradient  converges in linear rate to the global optimum.

 In addition, we note that in Algorithm \ref{algo:gtd} we assume that the initial state $x_0$ is sampled from the stationary distribution $\rho_K$. Such an  assumption is made only to simplify theoretical analysis. In practice, we could start the algorithm after sampling a sufficient number of transitions so that the Markov chain induced by $\pi_K$ approximately mixes. Moreover, as shown in \cite{tu2017least}, when $\pi_K$ is a stable policy such that $\rho(A- BK) < 1$, the Markov chain induced by $\pi_K$ is geometrically $\beta$-mixing and thus mixes rapidly.   
 
 Finally, we remark that the minimax  formulation of  the policy evaluation problem is first proposed in \cite{liu2015finite}, which studies the sample complexity of the GTD algorithm for discounted  MDPs with i.i.d. data.  
Using the same formulation,  \cite{wang2017finite} establishes finite sample bounds with   data  generated from a Markov process.  
Our optimization problem in \eqref{eq:new_opt} can be viewed as the extension of  their minimax formulation to the ergodic setting. Besides, our GTD algorithm can be applied to ergodic MDPs in general  with dependent data, which might be of independent interest.

\begin{algorithm} [ht]
	\caption{On-Policy Gradient-Based Temporal-Difference Algorithm for Policy Evaluation} 
	\label{algo:gtd} 
	\begin{algorithmic} 
		\STATE{{\textbf{Input:}} Policy $\pi_K$,    number of iterations $T$,  and  stepsizes   $\{ \alpha_t\}_{t\in [T] }$.}  
		\STATE{{\textbf{Output:} Estimator  $\hat \Theta $ of  $\Theta_K$  in \eqref{eq:ThetaK}.}}
	 	\STATE{Initialize the primal  and dual variables by $\vartheta_0 \in \cX_{\Theta} $ and $\omega_0\in \cX_{\Omega} $, respectively.}
	 	\STATE{Sample the  initial state $x_0\in \RR^d$ from the stationary distribution  $\rho_K$. Take action $u_0\sim \pi_K(\cdot \given x_0)$ and obtain the reward $c_0$ and the next state $x_1$.}
		\FOR{$t= 1 ,  2, \ldots,  T $}
		\STATE{Take action $u_{t}$ according to policy $\pi_K$, observe the reward $c_t$ and the next state $x_{t+1}$.}
		\STATE{Compute the TD-error $ \delta_t = \vartheta_{t-1}^1  - c_{t-1} +  [ \phi(x_{t-1}, u_{t-1}  ) - \phi(x_{t}, u_{t} )] ^\top \vartheta_{t-1} ^2$. }
		\STATE{Update   $\vartheta^1$ by $ \vartheta_{t} ^1 = \vartheta_{t-1} ^1 -  \alpha _t   \cdot [ \omega_{t-1} ^1 +   \phi(x_{t-1}, u_{t-1}) ^\top \omega_{t-1} ^2] .$} 
		\STATE{Update $\vartheta^2$  by   $ \vartheta_{t} ^2= \vartheta_{t-1}^2 -   \alpha _t   \cdot [  \phi(x_{t-1}, u_{t-1} )- \phi(x_t, u_t )]  \cdot \phi(x_{t-1}, u_{t-1}) ^\top \omega_{t-1}^2. $} 
		\STATE{Update  $\omega^1$  by $ \omega_{t}^1  = ( 1-\alpha_t ) \cdot  \omega_t ^1 + \alpha_t \cdot  ( \vartheta_{t-1}^1 - c_{t-1}) .$} 
		\STATE{Update $\omega^2$ by $ \omega_{t}^2 = ( 1 - \alpha_t ) \cdot \omega_t ^2 + \alpha_t \cdot \delta _t \cdot \phi(x_{t-1}) .$} 
		\STATE{Project $\vartheta_t $ and $\omega_t$ to $\cX_{\Theta}$ and $\cX_{\Omega}$, respectively.} 
		\ENDFOR
		\STATE{Define $\hat \vartheta =(\hat \vartheta^1, \hat \vartheta^2)  =   (  \sum_{t=1}^{T} \alpha_t \cdot \vartheta_t )/ (  \sum_{t=1}^{T} \alpha _t ) $ and $\hat \omega =  (  \sum_{t=1}^{T} \alpha_t \cdot \omega_t ) / (  \sum_{t=1}^{T} \alpha _t ) $}.
		 \STATE{Return $\hat \vartheta^1$ and $\hat \Theta = \smat(\hat \vartheta^2)$ as the estimators  of  $J(K)$ and  $\Theta_K$, respectively.}
	\end{algorithmic}
\end{algorithm}

\subsection{Extension to the Off-Policy Setting} \label{sec:offac}
Recall that in our natural actor-critic algorithm, the critic can apply any policy evaluation algorithm to estimate $\hat \Theta_K$. When using an off-policy method, we obtain an off-policy actor-critic algorithm.  In this section, we extend Algorithm \ref{algo:gtd}  to the off-policy setting using importance sampling.  Specifically, let $\pi_b$ be the behavior policy and suppose it induces a stationary distribution $\rho_b$ over the state space $\RR^d$. Moreover, let $\pi_K$ be the policy of interest  and let $\tau_K(x, u)  = \pi_K (u\given x) / \pi_b( u \given x) $ be the importance sampling ratio. Then, the  Bellman equation in \eqref{eq:new_bellmanQ} can be written  as 
\#\label{eq:offbellman}
\la \phi(x, u ), \theta_K ^*\ra =  c(x,u)  - J(K) + \bigl \la \EE [ \phi(x', u' )\cdot \tau_K(x', u') \given x, u], \theta_K ^* \bigr  \ra , \qquad \forall (x, u ) \in \RR^{d+k},
\#
where $x'$ is the next state given $(x,u)$, and $u' \sim \pi_b(\cdot \given x)$. In the following, we denote by $\EE_{(x,u)}$ the expectation with respect to $x\sim \rho_b$ and $u \sim \pi_b (\cdot \given x)$. Similar to $\Xi_K$ and $b_K$ defined in \eqref{eq:gtd_mat}, for the off-policy setting, we define 
\$
& \overline \Xi_K   = \EE_{(x, u)  }  \bigl \{  \phi (x, u )  \bigl [ \phi(x, u ) -  \tau_K(x', u') \cdot \phi (x', u')   \bigr ]^{\top } \bigr \}, \qquad  \overline  b_K = \EE_{(x ,u)  } \bigl [    c(x,u)    
\phi (x, u) \bigr ],   \notag \\
& \overline  h_K  = \EE _{(x,u)} [ \phi(x,u) - \tau_K(x', u') \cdot \phi (x', u')] , \qquad \overline g_K  =  \EE _{(x,u)} [ \phi(x,u) ] ,\qquad \overbar a_K = \EE_{(x,u)} [ c(x,u)].
\$
Based on \eqref{eq:offbellman} and direct computation, it can be shown that $\vartheta_K^* = (J (K), \svec(\Theta_K) ^\top) ^\top$ is the solution to linear equation 
\$ 
\begin{pmatrix}
1 & \overline h_K^\top 
\\
\overline  g_K  & \overline \Xi_K 
\end{pmatrix}
\begin{pmatrix}
\vartheta^1 \\
\vartheta ^2 
\end{pmatrix} = 
\begin{pmatrix}
\overline a_K \\
\overline b_K
\end{pmatrix}.
\$
Similar to the derivations in \S\ref{sec:ac}, we propose to estimate $\vartheta_K^*$ by solving a minimax optimization problem:
\#\label{eq:new_opt2}
\min_{\vartheta \in \cX_{\Theta} }\max _{\omega\in \cX_{\Omega} } \overline F(\vartheta, \omega) & =   [ \vartheta^1   + \overline h_K ^\top \vartheta^2 - \overline a_K  ]  \cdot \omega ^1    + \bigl \la  \vartheta^1 \cdot \overline g_K + \overline \Xi_K \vartheta^2 - \overline b_K , \omega^2 \bigr \ra   - 1/2 \cdot \| \omega\|_2^2  .
\#
Notice that  both $\overline F(\vartheta, \omega)$  and its gradient can be estimated unbiasedly using transitions sampled from  the behavior policy. Solving \eqref{eq:new_opt2} using stochastic gradient method, we obtain the off-policy GTD algorithm for the ergodic setting. Due to the similarity to Algorithm \ref{algo:gtd}, we defer the details of off-policy GTD  to Algorithm \ref{algo:off-policy} in the appendix. Combining this policy evaluation method with Algorithm \ref{algo:ac}, we establish the off-policy on-line natural actor-critic algorithm.


\section{Theoretical Results}\label{sec:theory}

In this section, we  establish the global convergence of the natural actor-critic algorithm. To this end, we first   focus on the problem of policy evaluation by assessing the finite sample  performance of the on-policy GTD algorithm.  

Note that only $\hat \Theta$ returned by the GTD algorithm is utilized in the natural actor-critic algorithm for policy update. Thus, in  the policy evaluation problem for  a linear policy   $\pi_K$, we only need to study the estimation error $\| \hat \Theta - \Theta_K\|_{\fro}^2$, which characterizes the closeness between the direction of  policy update in Algorithm \ref{algo:ac} and the true natural policy gradient.

 Furthermore, recall that we restrict the primal and dual variables respectively to compact sets $\cX_{\Theta}$ and $\cX_{\Omega}$ for algorithmic stability. We make the following assumption on $\cX_{\Theta}$ and $\cX_{\Omega}$.

\begin{assumption} \label{assume:pe}
Let $\pi_{K_0}$ be the initial policy in Algorithm \ref{algo:ac}. We assume that $\pi_{K_0}$ is a stable policy such that $\rho(A - BK_0) < 1$. Consider the policy evaluation problem for $\pi_K$. We assume that $J(K) \leq J(K_0)$. Moreover, let $\cX_{\Theta } $ and $ \cX_{\Omega}$ in \eqref{eq:new_opt}  be defined as 
 \#
\cX_{\Theta }  & = \bigl \{ \vartheta \colon  0 \leq \vartheta^1   \leq J(K_0),      \| \vartheta^2 \|_{2} \leq \tilde R_{\Theta}   \bigr \},  \label{eq:projectsets} \\ 
 \cX_{\Omega}  &    = \bigl \{ \omega  \colon | \omega^1 | \leq J(K_0),  \| \omega^2\|_{2 } \leq  ( 1+ \| K \|_{\fro}^2  ) ^2  \cdot  \tilde R_\Omega   \bigr \}.  \label{eq:projectsets2} \#
Here,  $\tilde R_\Theta$ and $\tilde R_{\Omega}$ are two parameters that does not depend on $K$. Specifically, we have 
 \#
 \tilde R_{\Theta} & =  \| Q \| _{\fro}+ \| R \|_{\fro} + \sqrt{ d }  / \sigma_{\min} (\Psi)  \cdot   ( \| A \|_{\fro} ^2 + \| B \|_{\fro}^2   )\cdot J(K_0 )  ,   \label{eq:radius_theta} \\
 \tilde R_\Omega & = C \cdot   \tilde R_{\Theta} \cdot \sigma_{\min}^{-2} (Q) \cdot [J(K_0)
]^2  ,\label{eq:radius_omega}
 \#
 where $C > 0$ is a constant.
\end{assumption}

The assumption that we have access to a stable policy $K_0$ for initialization is commonly made in  literatures on model-free methods for LQR  \citep{dean2018regret, fazel2018global, malik2018derivative}. Besides, $\rho(A - BK_0) < 1$ implies that $J(K_0)$ is finite. Here we assume $J(K) \leq J(K_0)$ for simplicity. 
Even if $J(K) > J(K_0)$, we can replace $J(K_0)$ in \eqref{eq:projectsets} -- \eqref{eq:radius_omega} by $J(K)$ and the  theory of policy evaluation still holds.
Moreover,  as  we will show in Theorem \ref{thm:ac}, the actor-critic algorithm creates a sequence policies whose objective values decreases monotonically. Thus, here we assume $J(K) \leq J(K_0)$ without loss of generality.

Furthermore, as shown in the proof, the construction of $\tilde R_{\Theta} $  and $\tilde R_{\Omega}$ ensures that $(\vartheta _K^* , 0)$ is the saddle-point of the minimax optimization in \eqref{eq:new_opt}. In other words, the solution to \eqref{eq:new_opt} is the same as the unconstrained problem $\min_{\vartheta} \max_{\omega} F(\vartheta, \omega)$. When replacing the population problem by a sample-based optimization problem, restrictions on the primal and dual variables ensures that the iterates of the GTD algorithm remains bounded. Thus, setting $\tilde R_{\Theta}$ and $\tilde R_{\Omega}$ essentially guarantees that restricting  $(\vartheta , \omega)$ to $\cX_{\Theta} \times \cX_{\Omega}$ incurs no ``bias'' in the optimization problem. 

We present the theoretical result for the online GTD algorithm as follows.
\begin{theorem} [Policy evaluation] \label{thm:pe}
Let $\hat \vartheta^1$ and $\hat \Theta$ be the output of Algorithm \ref{algo:gtd} based on $T$ iterations. We set the stepsize to be $\alpha_t = \alpha / \sqrt{t}$ with $\alpha > 0$ being a constant. 
Under Assumption \ref{assume:pe}, for any $\rho \in (\rho(A - BK), 1)$, when the number of iterations $T$ is sufficiently large, with probability at least    $1 - T^{-4} $, we have 
	\#\label{eq:stat_rate}
	  \| \hat \Theta - \Theta_K \|_{\fro}^2 \leq   \frac{  \Upsilon \bigl [   \tilde R_{\Theta} ,  \tilde R_{\Omega} , J(K_0), \| K \|_{\fro}, \sigma_{\min}^{-1}(Q)  \bigr ] }{   {\kappa_K^*}^2 \cdot  ( 1- \rho)  } \cdot  \frac{  \log ^6 T      }{ \sqrt{T}},
	\#
	where $  \Upsilon \bigl [   \tilde R_{\Theta} ,  \tilde R_{\Omega} , J(K_0), \| K \|_{\fro}, \sigma_{\min}^{-1}(Q)  \bigr ]$ is a polynomial of $  \tilde R_{\Omega}$, $  \tilde R_{\Omega}$, $J(K_0)$, $\| K\|_{\fro}$, and $ 1/  \sigma_{\min} (Q) $. 
\end{theorem}
\begin{proof} 
See \S\ref{proof:thm_pe} for a detailed proof.
\end{proof} 

This theorem establishes the statistical rate of convergence for the on-policy GTD algorithm. 
Specifically, if we regard $\Upsilon   [   \tilde R_{\Theta} ,  \tilde R_{\Omega} , J(K_0), \| K \|_{\fro}, \sigma_{\min}^{-1}(Q)   ]$, $\rho$, and $  \kappa_K^*$ as constant, \eqref{eq:stat_rate} implies that the estimation error is of order $\log ^6 T /\sqrt{T}$. 
Thus, ignoring the logarithmic term, we conclude that the GTD algorithm converges in the sublinear rate $\cO(1/\sqrt{T} )$, which is optimal for convex-concave stochastic optimization  \citep{nemirovski2009robust} and is also identical to the rate of convergence of the GTD algorithm in the discounted setting with bounded data \citep{liu2015finite, wang2017finite}. 
Note that we focus on the ergodic case and the feature mapping $\phi(x,u) $ defined  in \eqref{eq:feature_map} is unbounded. We believe this theorem might be of independent interest. 
Furthermore, $ 1/ {\kappa_K^*}$  is approximately the condition number of the linear equation of \eqref{eq:large_le}, which reflects the fundamental   difficulty of estimating $\Theta_{K}$. Specifically, when $ {\kappa_K^*}$ is close to zero, the matrix on the left-hand side of  \eqref{eq:large_le} is close to a singular matrix. In this case, estimating $\Theta_K$ can be viewed as  solving an ill-conditioned regression problem and thus  huge sample size is required for consistent estimation. Finally,  $1/ [1- \rho(A- BK) ] $ also reflects the intrinsic hardness of estimating $\Theta_K$. Specifically,  for any $\rho \in (\rho(A - BK), 1) $,  the Markov chain induced by $\pi_K$ is  $\beta$-mixing where the $k$-th mixing coefficients is bounded by  $C \cdot \rho^k$ for some constant $C >0 $      \citep{tu2017least}. Thus, when $\rho$ is close to one, this Markov chain becomes more dependent, which  makes the estimation problem more difficult.

Equipped with the finite sample error of the policy evaluation algorithm, now we are ready to present the global convergence of the actor-critic algorithm.  For ease of presentation, we assume that $Q$, $R$, $A$, $B$, $ \Psi$ are all constant matrices. 

\begin{theorem} [Global convergence of actor-critic] \label{thm:ac}
Let the initial policy $K_0$ be stable.  We set the stepsize $\gamma = [ \| R \| + \sigma_{\min}^{-1} (\Psi) \cdot \| B \|^2 \cdot J(K_0)   ] $  in Algorithm \ref{algo:ac} and perform $N$  actor updates in the actor-critic algorithm.  Let $\{ K_t\}_{0 \leq t\leq N}$ be the sequence of policy parameters generated by the algorithm. For any sufficiently small $\epsilon > 0$, we set $N > C \cdot \| \Sigma_{K^* } \| / \gamma \cdot   \log \bigl \{ 2 [    J(K_0) - J(K^*) ]  / \epsilon \bigr \} $ for some constant. Moreover, for any $t \in \{0, 1, \ldots, N\}$, in the $t$-th iteration, we set the number $T_t$  of GTD updates in Algorithm \ref{algo:gtd} to be 
\$
T_t \geq \Upsilon  \bigl  [ \| K_t \|, J(K_0) \bigr ] \cdot  {\kappa_{K_t} ^*} ^{- 5}   \cdot [ 1- \rho( A - BK_t) ]^{-5/2} \cdot \epsilon^{-5} ,
\$
where $ \Upsilon   [ \| K_t \|, J(K_0) ] $ is a polynomial of $\|K_t \|$ and $J(K_0)$. Then with probability at least $1 - \epsilon^{10}$, we have $J(K_N) - J(K^*) \leq \epsilon$.
\end{theorem}

\begin{proof}[Proof Sketch] 
The proof of this Theorem is based on combining the convergence of the natural policy gradient and the finite sample analysis of the GTD algorithm established in Theorem \ref{thm:pe}. Specifically, for each $K_t$, we define $K_{t+1}' = K_t - \eta \cdot E_{K_t}$, which is the one-step natural policy gradient update starting from $K_t$. Similar to \cite{fazel2018global}, for ergodic LQR, it can be shown that   \#\label{eq:linear_converge}
 J(K_{t+1}' ) - J(K^* ) \leq \bigl [ 1 -  C_1\cdot  \gamma \cdot \| \Sigma_{K^*} \|^{-1}  \bigr ] \cdot \bigl [ J(K_{t } ) - J(K^* )   \bigr ] \#
  for some constant $C_1 > 0 $.  In addition,  for policy $\pi_{K_t}$,  when the number of GTD iteration $T_t$ is sufficiently large, $K_{t+1}$ is close to $K_{t+1}'$, which further implies that  $ |  J(K_{t+1}' ) - J(K_{t+1}) | $ is  small. Thus, combining this and \eqref{eq:linear_converge}, we obtain the linear convergence of the actor-critic algorithm. See \S\ref{proof:thm_ac} for a detailed proof.
\end{proof}

This theorem shows that  natural actor-critic algorithm combined with GTD converges linearly to the optimal policy of LQR. Furthermore, the number of policy updates in this theorem matches those obtained by natural policy gradient  algorithm \citep{fazel2018global, malik2018derivative}.  
 To the best of our knowledge, this result seems to be the first nonasymptotic convergence result for actor-critic algorithms with function approximation, whose existing theory are mostly asymptotic and based on ODE approximation. 
 Furthermore, from the viewpoint of bilevel optimization, Theorem \eqref{thm:ac} offers theoretical guarantees for   the actor-critic algorithm as a first-order online method for the bilevel program defined in \eqref{eq:nac_bilevel}, which serves a first attempt of understanding bilevel optimization with possibly nonconvex subproblems.


\section{Proofs of the Main Results} \label{sec:proof_main}

In this section, we provide the proofs of the main results, namely, Theorems \ref{thm:pe} and \ref{thm:ac}, which are proved in \S\ref{proof:thm_pe} and \S\ref{proof:thm_ac}, respectively.
The proofs of the supporting results are deferred to the appendix.

\subsection{Proof of Theorem \ref{thm:pe}}\label{proof:thm_pe}
\begin{proof}
	  Our proof can be decomposed into three steps. In the first step, we show that, with $\cX_{\Theta}$ and $\cX_{\Omega}$ given in \eqref{eq:projectsets} and  \eqref{eq:projectsets2},   $(\vartheta , \omega) = (\vartheta_K^*, 0) $ is the solution to the minimax optimization problem in \eqref{eq:new_opt}. Then, in the second step, we show that the primal-dual gap of this optimization problem yields an upper bound for the  estimation error $\| \hat \Theta -   \Theta_K \|_{\fro}^2 $, where $\hat \Theta = \smat(\hat \vartheta^2)$ is the estimator of $\Theta_K$ returned by the GTD algorithm. Finally, in the last step, we study the performance of such a minimax optimization problem, which enables us to establish the error of policy evaluation.
	  
	  \vspace{5pt}
   {\noindent \textbf{Step 1.}} In the first step, we show that $  (\vartheta , \omega) = (\vartheta_K^*, 0) $ is the saddle point of the optimization problem in \eqref{eq:new_opt}.  
   To simplify the notation, we define a vector-valued function $G(x, u, x', u'; \vartheta)$ by  	\#\label{eq:define_Gfun}
	\begin{split}
    G^1 ( x, u, x', u';   \vartheta)  	& =   \vartheta^1 -  c(x,u),     \\
    G^2 ( x, u, x', u';   \vartheta)  & =  \vartheta^1 \cdot \phi(x,u) + \bigl\{  \bigl  [     \phi(x, u )-  \phi (x', u')   \bigr  ] ^\top \vartheta^2  - c(x,u)\bigr \} \cdot  \phi(x, u)  . 
    \end{split}
  \#
  By definition, $G(x, u, x', u'; \vartheta)$ is of the same shape as $\vartheta $ and $\omega$. Moreover, for all $(\vartheta , \omega)  $, $F(\vartheta , \omega) $ in \eqref{eq:new_opt} can be equivalently written as  
  \#\label{eq:new_F}
  F(\vartheta, \omega) = \bigl \la  \EE_{(x,u, x', u')} [   G(x, u, x', u' ; \vartheta) ] , \omega  \bigr \ra  - 1/2 \cdot \| \omega \|_2^2.
  \#
Thus, for any $\vartheta  $, the solution to the unconstrained maximization problem $\max_{\omega} F(\theta , \omega)$ is 
\#\label{eq:expected_G}
w(\vartheta) = \EE_{(x,u, x', u')} [   G(x, u, x', u' ; \vartheta) ]. 
\#   
   In the following, we  show that  $\vartheta_K^* \in \cX_{\Theta}$.  Moreover, we also prove  that, for any $\vartheta \in \cX_{\Theta}$, $w(\vartheta) $ in  \eqref{eq:expected_G}  belongs to $\cX_{\Omega}$, where $\cX_{\Theta}$ and $\cX_{\Omega}$ are defined in     \eqref{eq:projectsets} and  \eqref{eq:projectsets2}, respectively.   Since $w(\vartheta_K^* )  = 0$,  it holds that $(\vartheta_K^*, 0)$ is the solution to the minimax optimization problem in \eqref{eq:new_opt}.

 Recall  that we assume $J(K) \leq J(K_0)$, where $K_0$ is the initial policy that is stable. Thus,  $J(K_0)$ is finite.
 By the definition of $\vartheta_K^*$, to show  $\vartheta_K^* \in \cX_{\Theta}$, it suffices to bound  $    \| \Theta_K \|_{\fro} $.
By the definition of $\Theta_K$ in \eqref{eq:ThetaK}, we have 
		\$ 
	\Theta_K = 
	\begin{pmatrix}
		Q + A^\top P_K A  & A ^\top P_K B  \\
		B^\top P_K A   & R + B^\top P_K B 
	\end{pmatrix} 
	= \begin{pmatrix} 
	Q &   \\
	& R \end{pmatrix} + \begin{pmatrix}
	A^\top \\
	B ^\top \end{pmatrix}
	P_K 
	\begin{pmatrix}
	A  & B
	 \end{pmatrix} ,
	\$
	which implies that 
	\#\label{eq:bound_tknorm}
	\| \Theta_K \|_{\fro} \leq   (  \| Q \|_{\fro}  + \| R \|_{\fro} ) + ( \| A \|_{\fro} ^2 + \| B \|_{\fro}^2 ) \cdot \| P_K \| _{\fro}.
	\#
	Now we apply the following lemma to obtain an upper bound on $\| P_K \|_{\fro}$.
	
	\begin{lemma}\label{lemma:bound_mats}
	When $\pi_K$ is a stable policy, we have 
	\$
	\| \Sigma_K \| \leq J(K) / \sigma_{\min}(Q), \qquad \| P_K \| \leq J(K) / \sigma_{\min}(\Psi) ,
	\$
	where   $\sigma_{\min}(\cdot )$ denotes the minimal eigenvalue of a matrix.
\end{lemma}
\begin{proof}
By \eqref{eq:cost_K2} in Proposition \ref{prop:pg}, we have 
\$
J(K) & \geq \tr [ (Q+ K^\top R K ) \Sigma_K ] \geq \sigma_{\min}(Q) \cdot \tr ( \Sigma_K)  \geq \sigma_{\min}(Q) \cdot  \| \Sigma_K \| ,\\
J(K) & \geq \tr ( P_K \Psi_{\sigma} ) \geq \sigma_{\min}(\Psi_{\sigma} ) \cdot \tr( P_K) \geq \| P_K \| \geq J(K) / \sigma_{\min} (\Psi) ,
\$ 
where we use the fact that $\Psi_{\sigma} \succeq \Psi$. Therefore, we conclude the proof. 
\end{proof}

Applying Lemma \ref{lemma:bound_mats}
to \eqref{eq:bound_tknorm}, we have 
\#\label{eq:bound_tknorm2}
\| \Theta_K \|_{\fro} \leq (  \| Q \|_{\fro}  + \| R \|_{\fro} ) + ( \| A \|_{\fro} ^2 + \| B \|_{\fro}^2 ) \cdot \sqrt{d} \cdot J(K) / \sigma_{\min} (\Psi) .  
\#
Combining  \eqref{eq:bound_tknorm2} and the definition of $\tilde R_{\Theta}$  in \eqref{eq:radius_theta} we conclude that 
$\vartheta_K^* \in \cX_{\Theta} $. 
 
  Furthermore, it remains to show that the vector in \eqref{eq:expected_G} belongs to $\cX_{\Omega}$ for all $\vartheta \in \cX_{\Theta}$. We consider the two components of $G(x, u, x', u'; \vartheta) $ separately. By \eqref{eq:define_Gfun}, we have 
  \#\label{eq:omega_1bound}
 \bigl | \EE_{(x,u,x',u')} [ G^1 (x, u, x', u'; \vartheta)  ]  \bigr | =|  \vartheta^1 - J(K)  | \leq   J(K_0),
  \#
  where the second inequality follows from the fact that $0\leq \vartheta^1 \leq J(K_0)$.
 Moreover, by \eqref{eq:define_Gfun}, for the second component of  $G(x, u, x', u'; \vartheta) $, we have 
 \#\label{eq:omega_2bound}
 \EE_{(x,u,x',u')} [ G^2 (x, u, x', u'; \vartheta)  ]  =   \vartheta^1 \cdot \EE_{(x,u)}  [ \phi(x,u) ]  + \Xi_K \vartheta^2 - b_K, 
  \#
  where $\Xi_K$ and $b_K$ are defined in \eqref{eq:gtd_mat}. 
  By Lemma \ref{lemma:pe_mat}, we have 
  \#\label{eq:omega_2bound2}
  \| \Xi_K \vartheta^2 \| _2 \leq \| \Xi_K \| \cdot \|  \vartheta^2  \|_2 \leq  4 ( 1+ \| K \|_{\fro}^2  ) ^2  \cdot  \| \Sigma_K \|^2  \cdot \tilde R_{\Theta}.
  \#
  Moreover, for any positive definite matrix $\Gamma$, we have 
  \#\label{eq:omega_2bound3}
  b_K^\top \svec(\Gamma) = \EE _{(x, u)} \bigl \{  \bigl \la \phi(x,u) , \smat[ \diag(Q, R) ] \bigr  \ra  \cdot    \bigl \la \phi(x,u) , \smat(\Gamma) \bigr  \ra  \bigr  \},
  \#
  where $\diag(Q, R)$ is the block diagonal matrix constructed by $Q$ and $R$.
Note that the joint distribution of $(x,u)$ is the Gaussian distribution $N(0, \tilde \Sigma_K) $, where $\tilde \Sigma_K $ is defined in \eqref{eq:joint_cov}. Thus, $b_K^\top \svec(\Gamma) $ can be written as the product of two quadratic forms of Gaussian random variables.  Applying Lemma \ref{lem:quadform} to \eqref{eq:omega_2bound3}, we obtain that 
\$
b_K^\top \svec(\Gamma) = 2  \bigl \la \tilde \Sigma _K  \diag(Q, R) \tilde \Sigma _K , \Gamma   \bigr \ra \cdot     +   \bigl \la \tilde \Sigma _K ,  \diag(Q, R)  \bigr \ra   \cdot \bigl \la \tilde \Sigma _K ,  \Gamma \bigr \ra ,
\$
which implies that 
\#\label{eq:omega_2bound4}
\| b_K \|_2 \leq  3 ( \|Q \|_{\fro} + \| R \|_{\fro} )\cdot    \|\tilde \Sigma _K\|  ^2 .
 \#
In addition, the first term on the right-hand side of \eqref{eq:omega_2bound} is bounded by
\#\label{eq:omega_2bound5}
\bigl \| \vartheta^1 \cdot \EE_{(x,u)}  [ \phi(x,u) ] \bigr \|_2 \leq J(K_0) \cdot \bigl  \| \tilde \Sigma_K \bigr \|_{\fro} .
\#
Finally, combining \eqref{eq:omega_2bound2},  \eqref{eq:omega_2bound4}, \eqref{eq:omega_2bound5}, and the upper bounds  in \eqref{eq:bound_cov_norm}, 
we have 
\#\label{eq:omega_2bound6}
& \bigl \|  \EE_{(x,u,x',u')} [ G^2 (x, u, x', u'; \vartheta)  ]  \bigr \|_2  \notag \\
&\qquad \leq 2  ( d + \| K \|_{\fro}^2 ) \cdot \| \Sigma_K \| +  4 ( 1+ \| K \|_{\fro}^2  ) ^2  \cdot  \| \Sigma_K \|^2  \cdot \tilde R_{\Theta}  \notag \\
&\qquad \qquad \qquad + 12 ( \|Q \|_{\fro} + \| R \|_{\fro} )\cdot  ( d + \| K \|_{\fro}^2 ) ^2    \cdot \| \Sigma_K \|  ^2  \notag \\
& \qquad \leq C \cdot  ( 1+ \| K \|_{\fro}^2  ) ^2 \cdot \tilde R_{\Theta} \cdot \sigma_{\min}^{-2} (Q) \cdot [J(K_0)
]^2, 
\#
where $C>0$ is an absolute constant. 

Hence, combining \eqref{eq:radius_omega}, \eqref{eq:omega_1bound} and \eqref{eq:omega_2bound6}, we conclude that $ w(\vartheta)   \in \cX_{\Omega}$ for all $\vartheta \in \cX_{\Theta}$.  Therefore, we have shown that $(\vartheta_K^* , 0)$ is the saddle point of the optimization problem in \eqref{eq:new_opt}, which concludes the first step of the proof.

  \vspace{5pt}
   {\noindent \textbf{Step 2.}}  In the following, we relate the estimation error $\| \hat \Theta - \Theta_K \|_{\fro}^2$ to the performance  of  the optimization in \eqref{eq:new_opt}. Specifically, we consider the primal-dual gap
\#\label{eq:def_gap}
\texttt{Gap} (\hat \vartheta , \hat \omega) = \max_{\omega \in \cX_{\Omega}}  F( \hat \vartheta, \omega  )  - \min_{\vartheta \in \cX_{\Theta}} F (\vartheta, \hat \omega)  ,
\#
which characterizes the closeness between $( \hat \vartheta, \hat \phi)$ and the optimal solution$(\vartheta_K^* , 0)$, quantified by the objective value. 

Recall that $w(\vartheta) $ defined in \eqref{eq:expected_G} is the optimal dual variable for each $\theta \in \cX_{\Theta}$. Hence, for any $\omega \in \cX_{\Omega}$, it holds that 
 \#\label{eq:trash42}
    \min _{\vartheta \in \cX_{\Theta}} F (\vartheta, \omega ) & \leq   \min _{\theta \in \cX_{\Theta}}  \max_{\omega \in\cX_{\Omega} } F (\theta, \omega )  \notag \\
   &  \leq \min_{\vartheta \in \cX_{\Omega} } \bigl \{ [   \vartheta^1 - J(K)  ]^2 +   \|  \vartheta^1 \cdot \EE_{(x,u)}  [ \phi(x,u) ]  + \Xi_K \vartheta^2 - b_K \|_2^2 \bigr \} = 0 .
    \#
Thus, for $\hat \vartheta$ returned by the GTD algorithm, we have 
 \#\label{eq:error_gtd}
 & \bigl \{ [  \hat  \vartheta^1 - J(K)  ]^2 +   \|  \hat \vartheta^1 \cdot \EE_{(x,u)}  [ \phi(x,u) ]  + \Xi_K \hat  \vartheta^2 - b_K \|_2^2 \bigr \} =  \max _{\omega \in \cX_{\Omega} } F (\hat \vartheta, \omega )   \notag \\
 & \qquad =   \max _{\omega \in \cX_{\Omega} }F  (\hat \vartheta, \omega   )   -      \min _{\vartheta \in \cX_{\Theta}} F  (\vartheta, \hat \omega  )   +  \min _{\vartheta \in \cX_{\Theta}} F (\vartheta,  \hat \omega ) \leq \texttt{Gap} (\hat \vartheta , \hat \omega)  ,
   \#
   where the last inequality follows from \eqref{eq:trash42}.

Furthermore, by direct computation, we can bound the left-hand side of \eqref{eq:error_gtd} via 
\#\label{eq:lower_bound}
& \biggl \|  
 \begin{pmatrix} 
 1 & 0\\
\EE_{(x,u)} [ \phi(x,u)]  & \Xi_K 
 \end{pmatrix} (\hat \vartheta  - \vartheta_K^* )    \biggr \|_2 ^2 \notag \\
 &\qquad \geq {\kappa_K^*}^2  \cdot \| \hat \vartheta - \vartheta_K^* \|_2 ^2=   {\kappa_K^*}^  \cdot \bigl [ \| \hat \Theta - \Theta_K \|_{\fro}^2  + | \hat \vartheta^1 - J(K) |^2 \bigr ],
\#
where we utilize the fact that $\vartheta_K^*$ is the solution to the linear equation in \eqref{eq:large_le} and $\kappa_K^*$ is specified in Lemma \ref{lemma:pe_mat1}.
Therefore, combining \eqref{eq:error_gtd} and \eqref{eq:lower_bound}, we have 
\#\label{eq:step2}
 | \hat \vartheta^1 - J(K) |^2 + \| \hat \Theta - \Theta_K \|_{\fro}^2 \leq  {\kappa_K^*}^{-2} \cdot \texttt{Gap} (\hat \vartheta , \hat \omega),
\#
which establishes the connection between $ \| \hat \Theta - \Theta_K \|_{\fro}^2 $ and the primal-dual gap in \eqref{eq:def_gap}.

\vspace{5pt}
   {\noindent \textbf{Step 3.}}  In the last step, we construct an upper bound for the primal-dual gap. By \eqref{eq:step2}, this yields an upper bound for the error of parameter estimation. 
   
 Note that the distribution of the state-action  pair $(x, u)$ have unbounded support. We first construct an event such that $\{ \phi(x_t, u_t)\}_{t= 0}^T$ are  bounded conditioning on this event. To this end, we 
	establish an upper bound for tail probability of the  $\| \phi(x,u) \|_2 $ using  the   Hansen-Wright inequality stated as follows.
 
	\begin{lemma}[Hansen-Wright inequality] \label{lemma:hwtail} 
		For any integer $m>0$, let $A  $ be a matrix in $\RR^{m\times m}$ and let $\eta \sim N(0, I_m)$ be the standard Gaussian random variable in $\RR^m$. Then, there exists an absolute constant $C> 0$ such that, for any $t\geq 0$, we have 
		\$
		\PP\bigl [ \bigl | \eta ^\top A \eta - \EE(\eta ^\top A \eta) \bigr | > t  \bigr ] \leq 2\cdot  \exp \bigl [ -C \cdot \min ( t^2 \cdot \| A \|_{\fro}^{-2} , ~ t \cdot \| A \|^{-1}) \bigr ] 
		\$ 
	\end{lemma}
	\begin{proof}
		See \cite{rudelson2013hanson} for a detailed proof.
	\end{proof}

 Applying  Lemma \ref{lemma:hwtail}  to $(x, u) \sim N(0, \tilde \Sigma_K )$ with  $\tilde \Sigma_K $ defined in \eqref{eq:joint_cov}, we obtain
	\#\label{eq:applyhansen}
		\PP\bigl [   \bigl | \| x \|_2^2 + \| u\|_2^2 - \tr \bigl(\tilde \Sigma_K \bigr ) \bigr |  > t \bigr ]  \leq 2 \cdot  \exp \bigl [ -C \cdot \min \bigl  ( t^2 \cdot \bigl \| \tilde \Sigma_K  \bigr \|_{\fro}^{-2} , ~ t \cdot \bigl \| \tilde \Sigma_K  \bigr \|^{-1} \bigr ) \bigr ] .
	\#
	Setting $t = C_1 \cdot \log T \cdot \bigl \| \tilde \Sigma _K \|$ in \eqref{eq:applyhansen} with constant $C_1$ sufficiently large, it holds that  
	\#\label{eq:tail_t}
 t^2 \cdot \bigl \| \tilde \Sigma_K\bigr \|_{\fro}^{-2} =  \bigl \| \tilde \Sigma_K  \bigr \|_{\fro}^{-2}  \cdot C_1^2 \cdot \log^2  T \cdot \|\tilde \Sigma_K  \bigr \|^2    \geq C_1^2 \cdot (d+k)^{-1}\cdot \log ^2 T \geq   t \cdot \bigl \| \tilde \Sigma_K  \bigr \|^{-1},
	\#
	where the first inequality follows from the relation between the operator and Frobenius  norms, and the second inequality holds when $\log T \geq C_1 ^{-1}\cdot (d+k) $. 
	For ease of presentation, for any $t \in \{0, 1, \ldots, T \}$, we define 
	\#\label{eq:eventt}
	\cE_t = \Bigl \{   \bigl | \| x_t\|_2^2 + \| u_t\|_2^2 - \tr \bigl ( \tilde \Sigma_K \bigr ) \bigr | \leq C_1 \cdot \log T\cdot \bigl  \| \tilde \Sigma_K  \bigr \|    \Bigr \},
	\#
	and write $\cE   = \bigcap _{0\leq t\leq T} \cE_t$. 
	Combining \eqref{eq:applyhansen} and \eqref{eq:tail_t}, we obtain that $\cE_t $ holds with probability at least $1 - T^{-6}$.
	Thus, by 
   taking a union bound for 
 $\{ (x_t, u_t)\}_{t= 0}^T$, we have $\PP(\cE) \geq 1 -  2T^{-5}$.
Moreover,  combining  	\eqref{eq:eventt}  and \eqref{eq:bound_cov_norm} further implies that, on event $\cE$, we have 
	\#\label{eq:some_trash00}
	& 	\max_{0\leq t\leq T}  \bigl \{ \| x_t\|_2^2 + \|  u_t\|_2^2 \bigr \}    \leq  C_1 \cdot \log T\cdot \bigl  \| \tilde \Sigma_K  \bigr \| + \tr \bigl ( \tilde \Sigma_K \bigr ) \leq \bigl ( C_1 \cdot \log T + d+k \bigr ) \cdot  \bigl  \| \tilde \Sigma_K  \bigr \| \notag \\
		& \qquad \leq 2 C_1 \cdot \log T \cdot  \bigl  \| \tilde \Sigma_K  \bigr \| \leq 2  C_1 \cdot \log T \cdot  \big [ \sigma^2 + ( 1+ \| K \|_{\fro}^2  ) \cdot  \| \Sigma_K \| \bigr ].
	\#

	In the sequel, we study  the stochastic optimization problem in \eqref{eq:new_opt} with the restriction that $\cE$ holds.
	Specifically, for any state-action pair $(x,u)$, we define the truncated feature function as  
	\#\label{eq:tildephi}
	\tilde \phi(x, u) = \phi( x, u) \cdot \ind \Bigl \{   \bigl |   \| \phi(x, u) \|_2 ^2 - \tr (\tilde \Sigma_K ) \bigr |  \leq C_1 \cdot \log T\cdot \bigl  \| \tilde \Sigma_K  \bigr \|    \Bigr \}. 
	\# 
	By this definition, for any $t\in \{0, \ldots, t\}$, we have $\tilde \phi(x_t, u_t ) = \phi (x_t, u_t ) \cdot \ind_{\cE_t}$. Now we replace $\phi(x,u)$ by $\tilde \phi(x,u)$ in \eqref{eq:new_opt} and consider  the following minimiax optimization problem:
	\#\label{eq:opt_trunc}
	\min _{\vartheta \in \cX_{\Theta} } \max _{\omega \in \cX_{\Omega} } \tilde F (\vartheta, \omega  )  & =  \bigl  \la    \EE_{(x, u, x', u') } \bigl [ \tilde G( x, u, x', u';   \vartheta) \bigr ]   , \omega \bigr \ra - 1/2 \cdot \| \omega \|_2^2 ,
	\#
	where, similar to $  G( x, u, x', u';  \vartheta)$ in \eqref{eq:define_Gfun}, we define   $ \tilde G( x, u, x', u';  \vartheta) $   by 
	\#\label{eq:defin_tG}
	\begin{split}
    \tilde G^1 ( x, u, x', u';   \vartheta)  	& =   \vartheta^1 - \tilde c(x,u)  ,     \\
    \tilde G^2 ( x, u, x', u';   \vartheta)  & =  \vartheta^1 \cdot \tilde  \phi(x,u) + \bigl\{  \bigl  [     \tilde  \phi(x, u )-  \tilde  \phi (x', u')   \bigr  ] ^\top \vartheta^2  - \tilde c(x,u)  \bigr \} \cdot  \tilde \phi(x, u)  .
    \end{split} 
  \#
  Here we denote $\tilde c(x,u)  =   \la \tilde \phi(x, u), \svec[ \diag(Q, R) ]     \ra$ in \eqref{eq:defin_tG} to simplify the notation.

	We remark that, when $\cE$ is true, $(\hat \vartheta, \hat \omega)$  is also the solution returned by the gradient-based algorithm for the minimax optimization problem in \eqref{eq:opt_trunc}. As a result, when  $\cE$ holds, the primal-dual gap of \eqref{eq:opt_trunc} is equal to  
	$\max_{\omega \in \cX_{\Omega}}  \tilde F( \hat \vartheta, \omega )  - \min_{\vartheta \in \cX_{\Theta}} \tilde F ( \vartheta, \hat \omega )$.
	
In the following, we characterize the difference between the objective functions in \eqref{eq:new_opt} and \eqref{eq:opt_trunc}. For any $(\vartheta, \omega)\in \cX_{\Theta} \times \cX_{\Omega}$, by \eqref{eq:new_F} and \eqref{eq:opt_trunc} we have 
	\#\label{eq:diff_obj}
	 \bigl | F (\vartheta, \omega   ) - \tilde F(\vartheta, \omega  ) \bigr |  & = \bigl |   \bigl \la  \EE_{(x, u, x', u') } \bigl [G( x, u, x', u'; \vartheta)  - \tilde G( x, u, x', u';  \vartheta) \bigr ] ,   \omega  \bigr \ra   \bigr |  \notag \\
   &  \leq   \bigl |   \EE_{(x, u, x', u') } \bigl [G^1 ( x, u, x', u';   \vartheta)  - \tilde G^1( x, u, x', u';  \vartheta) \bigr ] \bigr | \cdot J(K_0)  \notag \\
   & \qquad \qquad + \bigl \|  \EE_{(x, u, x', u') } \bigl [G^2 ( x, u, x', u';   \vartheta)  - \tilde G^2 ( x, u, x', u';  \vartheta) \bigr ] \bigr \|_2 \cdot \tilde R_{\Omega}.
	\#
 By the definitions of  $G( x, u, x', u';   \vartheta) $ and  $ \tilde G( x, u, x', u';   \vartheta) $ in \eqref{eq:define_Gfun} and  \eqref{eq:defin_tG}, we have 
 \#\label{eq:computeG}
 G^1( x, u, x', u';   \vartheta)  - \tilde G^1 ( x, u, x', u';  \vartheta) & = c(x,u)  \cdot \ind _{\cA ^c}  \\
 G^1( x, u, x', u';   \vartheta)  - \tilde G^1 ( x, u, x', u';  \vartheta) & =  G^2 ( x, u, x', u';   \vartheta) \cdot \ind _{\cA ^c} + \phi(x', u') ^\top \vartheta ^2\cdot \phi(x,u) \cdot \ind_{\cA} \cdot  \ind_{\cB^c} \notag ,
 \#
 where  we denote   
 $
  \{     |   \| \phi(x, u) \|_2 ^2 - \tr (\tilde \Sigma_K )   |  \leq C_1 \cdot \log T\cdot  \| \tilde \Sigma_K    \|     \}
 $ 	
 and $
   \{     |   \| \phi(x, u) \|_2 ^2 - \tr (\tilde \Sigma_K )   |  \leq C_1 \cdot \log T\cdot  \| \tilde \Sigma_K    \|     \}$
   by $\cA$ and $\cB$, respectively, and $\cA^c$, $\cB^c$ are the complement sets of $\cA$ and $\cB$.
   
   For the first term on the right-hand side of \eqref{eq:diff_obj}, Cauchy-Schwarz inequality implies that 
   \#\label{eq:trash51}
       \bigl |   \EE_{(x, u, x', u') } \bigl [G^1 ( x, u, x', u';   \vartheta)  - \tilde G^1( x, u, x', u';  \vartheta) \bigr ] \bigr | \leq      \sqrt{ \PP(\cA^c)} \cdot \sqrt{ \EE[ c^2(x, u) ] }.
   \#
   Since $c (x,u)$  is a quadratic form of a Gaussian random variable, by Lemma \ref{lem:quadform}, we have 
   \$
   \EE[ c^2(x,u) ] & = 2 \tr  \bigl [ \tilde \Sigma_K \diag(Q, R)  \tilde \Sigma_K  \diag(Q, R) \bigr  ] + \bigl \{ \tr \bigl [  \tilde \Sigma_K \diag(Q, R) \bigr ]\bigr \}^2  \notag \\
   & \leq 3 ( \| Q \|_{\fro} + \| R \|_{\fro} )^2 \cdot  \| \tilde \Sigma_K\|_{\fro}^2  \leq 3 ( \| Q \|_{\fro} + \| R \|_{\fro} )^2 \cdot \bigl [ \sigma^2 \cdot k +  (d + \| K \|_{\fro}^2 ) ^2 \cdot \| \Sigma_K \| ^2 \bigr ], 
   \$
   where the last inequality follows from \eqref{eq:bound_cov_norm}.
   Besides, for the second term on the right-hand side of \eqref{eq:diff_obj}, 
    combining \eqref{eq:diff_obj}, \eqref{eq:computeG},   triangle inequality, and Cauchy-Schwarz inequality, we have 
    \#\label{eq:apply_cauchy}
   &    \bigl \|  \EE_{(x, u, x', u') } \bigl [G^2 ( x, u, x', u';   \vartheta)  - \tilde G^2 ( x, u, x', u';  \vartheta) \bigr ] \bigr \|_2  \notag \\
     &\qquad  \leq \Bigl \{ \bigl \|  \EE_{(x,u, x', u') }[   G^2 ( x, u, x', u';   \vartheta) \cdot   \ind _{\cA ^c}   ] \bigr \|_2 +  \bigl \|  \EE_{(x,u, x', u') }[  \phi(x', u') ^\top \vartheta ^2 \cdot \phi(x,u)    \ind _{\cB ^c}   ] \bigr \|_2 \Bigr \}  \notag \\
     & \qquad \leq \Bigl \{  \sqrt{ \PP(\cA^c)} \cdot \sqrt{ \EE\bigl  [ \bigl \| G^2 ( x, u, x', u';   \vartheta) \bigr  \|_2^2   \bigr ]  } + \sqrt{ \PP(\cB^c) } \cdot \sqrt{ \EE \bigl [  \bigl \|  \phi(x,u)  \cdot \phi(x', u') ^\top \vartheta ^2 \bigr \| _2^2   \bigr ]    } \Bigr \}.
    \#
    For the expectations on the right-hand side of \eqref{eq:apply_cauchy}, using the inequality $(a+b)^2 \leq 2a^2 + 2b^2 $, we have 
\#\label{eq:bound_G_moment1}
&  \EE\bigl  [ \bigl \| G^2 ( x, u, x', u';   \vartheta) \bigr  \|_2^2   \bigr ] \notag \\
 & \qquad  \leq 2 \cdot  \EE \Bigl \{ \bigl [  \vartheta^1 - c(x,u)  + \phi(x,u)^\top \vartheta^2  \bigr]^2 \cdot \| \phi(x, u)   \|_2^2 \Bigl  \} + 2 \cdot \EE \bigl [  \bigl \|  \phi(x,u)  \cdot \phi(x', u') ^\top \vartheta^2 \bigr \| _2^2   \bigr ].
\#
Further applying Cauchy-Schwarz inequality to \eqref{eq:bound_G_moment1}, we have 
\#
& \EE \Bigl \{ \bigl [  \vartheta^1 - c(x,u)  + \phi(x,u)^\top \vartheta^2  \bigr]^2 \cdot \| \phi(x, u)   \|_2^2 \Bigl  \}\notag \\
& \qquad  \leq  \Bigl (\EE\bigl \{ \bigl [  \vartheta^1 - c(x,u)  + \phi(x,u)^\top \vartheta^2  \bigr]   ^4  \bigl  \} \cdot \EE \bigl [  \| \phi(x, u)   \|_2^4 \bigr ] \Bigr )^{1/2}, \label{eq:bound_G_moment2}\\
&  \EE \bigl [  \bigl \|  \phi(x,u)  \cdot \phi(x', u') ^\top \vartheta \bigr \| _2^2   \bigr ] \leq \Bigl ( \EE  \bigl [ | \phi(x', u')^\top \vartheta | ^4 \bigr ]   \cdot   \EE \bigl [  \| \phi(x, u)   \|_2^4 \bigr ] \Bigr )^{1/2}. \label{eq:bound_G_moment3}
\#
Since the marginal distributions of $(x,u)$ and $(x', u')$ are both $N(0, \tilde \Sigma_K)$, in  \eqref{eq:bound_G_moment2} and \eqref{eq:bound_G_moment3} we bound the two terms in \eqref{eq:bound_G_moment1} using  the fourth moments of $N(0, \tilde \Sigma_K)$, which can be written as a polynomial of $J(K_0)$, $\| K \|_{\fro}$, $\| Q\|$, $\|R\|$, $\tilde R_{\Theta}$, and $\tilde R_{\Omega}$.   

Meanwhile, recall that we have shown that $\PP (\cA^c)   \leq T^{-6}$ and  $ \PP( \cB^c)   \leq T^{-6}$.
    Thus, when $T$ is sufficiently large, by combining \eqref{eq:diff_obj}, \eqref{eq:trash51},  \eqref{eq:apply_cauchy}, and \eqref{eq:bound_G_moment1}, we have 
$
  | F (\vartheta, \omega   ) - \tilde F(\vartheta, \omega  ) |   \leq  1/ T,
$ 
which implies that 
\#\label{eq:gap_diff}
 & \Bigl|    \texttt{Gap} (\hat \vartheta , \hat \omega) -  \Bigl [   \max_{\omega \in \cX_{\Omega}}  \tilde F( \hat \vartheta, \omega  )  - \min_{\vartheta \in \cX_{\Theta}} \tilde F ( \vartheta, \hat \omega )  \Bigr ] \Bigr | \notag \\
 & \qquad \leq   
 \max  _{\omega \in \cX_{\Omega}}  \bigl | F( \hat \vartheta, \omega  ) -    \tilde F( \hat \vartheta, \omega  ) \bigr | + \max _{\vartheta \in \cX_{\Theta}}  \bigl |  F ( \vartheta, \hat \omega )  -  \tilde F ( \vartheta, \hat \omega )  \bigr | \leq 
 \frac{2}{T}. 
\#

	Hereafter, we study the primal-dual gap in \eqref{eq:def_gap} conditioning on event $\cE$. 
	To simplify the notation, we define function $H  (\vartheta, \omega;  \phi, \phi' ) $ on $\cX_{\Theta} \times \cX_{\Omega}$  by 
	\$
	H  (\vartheta, \omega;   \phi, \phi'  ) &= \bigl \la   \tilde G(x, u, x', u'; \vartheta), \omega \bigr \ra - 1/2 \cdot \| \omega\|_2^2,
	\$
	where the function $\tilde \phi(x, u)$ is defined in \eqref{eq:tildephi}, and we  denote $  \tilde \phi(x,u)$ and $ \tilde \phi(x', u')$ by $\phi$ and $\phi'$, respectively.
	Using this definition, 
  the objective function $\tilde F (\vartheta, \omega ) $ in \eqref{eq:opt_trunc} can be written as $\tilde F (\vartheta, \omega )   = \EE_{(x, u , x', u')} [   H (\vartheta, \omega;   \phi, \phi'  ) ],  $ where $(x,u)$ and $(x', u')$ are two consecutive state-action pairs. Note that $H (\vartheta, \omega;   \phi, \phi'  )$ is a quadratic function of $(\vartheta, \omega)$ for all $\phi$ and $\phi'$. The partial gradients  of $H (\vartheta, \omega;   \phi, \phi'  )$ are given by
    \#
  \nabla_{\vartheta^1} H (\vartheta, \omega;   \phi, \phi'  ) & =  \omega^1 + \tilde \phi(x, u ) ^\top \omega^2, \label{eq:gradG1}  \\
  \nabla_{\vartheta^2} H (\vartheta, \omega;   \phi, \phi'  )  & =  [ \tilde  \phi (x, u) ^\top \omega^2  ] \cdot  [\tilde\phi (x , u )  - \tilde\phi(x', u' ) ] , \label{eq:gradG2} \\
   \nabla_{\omega^1} H (\vartheta, \omega;   \phi, \phi'  )  & = \vartheta^1 - \tilde c(x, u)   - \omega^1, \label{eq:gradG3} \\
   \nabla_{\omega^2} H (\vartheta, \omega;   \phi, \phi'  )  & = \tilde G^2 (x, u, x' , u' ; \vartheta)     - \omega^2.  \label{eq:gradG4} 
  \#
  By combining \eqref{eq:tildephi}, \eqref{eq:gradG1}, and \eqref{eq:gradG2},   we can bound the norm  of $ \nabla_{\vartheta} H(\vartheta, \omega;   \phi, \phi'  ) $ by
  \#\label{eq:grad_norm_bound1}
 \|   \nabla_{\vartheta} H (\vartheta, \omega;   \phi, \phi'  ) \bigr \|_2 & \leq   |  \omega^1 + \tilde \phi(x, u ) ^\top \omega^2 |+ \bigl \|  [ \tilde  \phi (x, u) ^\top \omega^2  ] \cdot  [\tilde\phi (x , u )  - \tilde\phi(x', u' ) ] \bigr \|_2    \\
 & \leq  |  \omega^1 | + 2 \|  \tilde \phi(x, u ) \|_2 \cdot \| \omega^2 \|_2 \cdot \bigl [  \| \tilde  \phi (x, u) \|_2 + \| \tilde\phi(x', u' ) \|_2  \bigr ] \notag \\
&   \leq J(K_0) +16  C_1  ^2 \cdot ( 1+ \| K \|_{\fro}^2  ) ^2 \cdot   \log^2  T  \cdot \bigl [ \sigma^2 + ( 1+ \| K \|_{\fro}^2  )   \cdot \|  \Sigma_K  \|        \bigr ]^2   \cdot  \tilde R_{\Omega}. \notag 
  \#
Here the second inequality holds when $   \| \tilde  \phi (x, u) \|_2 \geq 1$ and the last inequality follows from \eqref{eq:some_trash00}. 
  Similarly, combining triangle inequality, \eqref{eq:gradG3}, and  \eqref{eq:gradG4}, we have
      \#\label{eq:grad_norm_bound2}
  \bigl \|   \nabla_{\omega} H(\vartheta, \omega;   \phi, \phi'  )  \bigr \|_2 &  \leq| \vartheta^1 - \tilde c(x, u)   - \omega^1|  +    +  \bigl [ ( \|Q \|_{\fro} + \|R \|_{\fro} ) \cdot  \| \tilde \phi (x,u) \|_2  \notag \\
  &\qquad   + (  \|\tilde \phi (x', u')  \|_2   + \|  \tilde  \phi(x,u)  \|_2 ) \cdot  \tilde R_{\Theta} \bigr ]  \cdot \| \tilde  \phi(x,u) \|_2   \notag \\
  & \leq   2 J(K_0)+  16  C_1^2 \cdot  \log ^2   T \cdot  \bigl [ \sigma^2 + ( 1+ \| K \|_{\fro}^2  )    \cdot \|  \Sigma_K  \|        \bigr ]^2 \cdot \tilde R_{\Theta}   .
  \#
  where the last equality holds since $\tilde R_{\Theta} \geq \| Q \|_{\fro} + \| R \|_{\fro}$.
  Moreover, we have $\nabla^2_{\vartheta\vartheta} H(\vartheta, \omega;   \phi, \phi'  ) = 0$ and $- \nabla^2_{\omega \omega } H (\vartheta, \omega;   \phi, \phi'  ) $ is the identity matrix.

  We utilize the following lemma, obtained from \cite{tu2017least}, to handle the dependence along the trajectory.

  \begin{lemma}
  	[Geometrically $\beta$-mixing] \label{lemma:lds_mix}Consider a linear dynamical system  $X_{t+1} = L X_t + \varepsilon$, where $\{ X_t\}_{t\geq 0 } \subseteq \RR^m$,  $\varepsilon \sim N(0, \Psi)$ is the random noise, and $L\in \RR^{m\times m}$ has spectral radius smaller than one. We denote by $\nu_t$ the marginal distribution of $X_t$ for all $t\geq 0$. Besides, the    stationary distribution of this Markov chain is  denoted by    $N(0, \Sigma_{\infty})$. For any integer $k \geq 1$, we define the $k$-th mixing coefficient as 
  	\$ 
  	\beta(k) = \sup_{t\geq 0} \EE_{x \sim \nu_t} \bigl [  \bigl \| \PP_{X_k}(\cdot \given X_0 = x) - \PP_{N(0, \Sigma_{\infty})} (\cdot ) \bigr \|_{\text{TV}} \bigr ].
  	\$
  	Furthermore, for any $\rho \in (\rho(L), 1)$ and any $k\geq 1$,  we have 
  	\$
  	\beta(k) \leq C_{\rho, L}\cdot \bigl [ \tr (\Sigma_{\infty})+ m  \cdot (1 - \rho)^{-2}  \bigr ]^{1/2} \cdot \rho^k,
  	\$
  	where $C_{\rho, L}$ is a constant that solely depends on $\rho$ and $A$. That is, $\{ X_t\}_{t\geq 0}$ is geometrically $\beta$-mixing.
  	\end{lemma}
  \begin{proof}
  See Proposition 3.1 in \cite{tu2017least} for a detailed proof.
  \end{proof}
Recall that  under policy $\pi_K$, $\{ (x_t, u_t)\}_{t\geq 0}$ form a linear dynamic system characterized by  \eqref{eq:joint_dynprime} and \eqref{eq:large_lds}. Since $\rho(L ) = \rho(A- BK) < 1$,  Lemma \ref{lemma:lds_mix} implies that, for all $\rho \in (\rho(A- BK), 1)$, ${(x_t, u_t) }_{t\geq 0}$ is a geometrically $\beta$-mixing stochastic process with parameter $\rho$.
  The following theorem, adapted from Theorem 1 in \cite{wang2017finite}, establishes the primal-dual gap for a  convex-concave  minimax optimization problem involving  a geometrically $\beta$-mixing stochastic process. 

	\begin{theorem}[Primal-dual gap for minimax optimization] \label{thm:saddle_gap}
		Let $\cX$ and $\cY$ are bounded and closed convex sets  such that $ \| x - x' \|_2 \leq D$ for all $x, x' \in \cX$ and $\| y - y'\|_2\leq D$ for all $y, y' \in \cY$. Consider the gradient algorithm for stochastic minimax optimization problem 
		\#\label{eq:thm_saddle}
		\min_{x\in \cX} \max _{y \in \cY} F(x,y ) = \EE_{\xi\sim \pi_{\xi} } [ \Phi(x, y ; \xi)],
		\#
		where $\xi $ is a random variable with distribution  $\pi_{\xi}$ and $F(x,y)$ is convex in $x$ and concave in $y$.  In addition, we assume that $\pi_{\xi}$ is the stationary distribution of a Markov chain $\{ \xi_t\}_{t\geq 0}$ which is geometrically $\beta$-mixing with parameter $\rho \in (0,1)$. Specifically, we assume that there exists a constant $C_{\xi} > 0$ such that, for all $k\geq 1$, the $k$-th mixing coefficient satisfy $\beta(k) \leq C_{\xi} \cdot \rho^k$.   Furthermore, we consider the case where, almost surely for every $\xi \sim \pi_{\xi}$,  $\Phi(x,y; \xi)$ is $L_1$-Lipschitz in both $x$ and $y$, $\nabla_{x} \Phi(x,y; \xi) $ is $L_2$-Lipschitz in $x$ for all $y\in \cY$, and  $\nabla_{y} \Phi(x,y; \xi)$ is $L_2$-Lipschitz in $y$ for all $x\in \cX$. Here, without loss of generality, we assume that $D, L_1, L_2 >1$.
		Consider solving the optimization problem in \eqref{eq:thm_saddle} via $T$ iterations of the gradient-based updates
		\$
		x_{t } = \Pi_{\cX} \bigl [  x_{t-1} -  \alpha_t \nabla _{x} \Phi(x_{t-1} , y_{t-1}; \xi_{t-1}) \bigr ] , \qquad y_{t} = \Pi_{\cY} \bigl [y_{t-1} +  \alpha_t \cdot \nabla_y \Phi(x_{t-1}, y_{t-1}; \xi_{t-1}  ) \bigr ], 
		\$
		where $t\in [T]$,  $\Pi_{\cX}$ and $\Pi_{\cY}$ are projection operators, and $\{ \alpha_t  = \alpha / \sqrt{t}\}_{t\in[T]}$ are the stepsizes, where $\alpha>0$ is a constant.   
		Let    
		$$ \hat x =  \frac   { \sum_{t\in [T] } \alpha_t \cdot x_t} {\sum_{t\in [T] }\alpha_t} , \qquad  \hat y =   \frac   { \sum_{t\in [T] } \alpha_t \cdot y_t} {\sum_{t\in [T] }\alpha_t}  $$ 
			be the final output of the algorithm.
		Then, there exists an absolute constant $C> 0$ such that, for any $\delta \in (0, 1)$, with probability at least $1-\delta$, the primal-dual gap  satisfies
		\$
		\max_{y \in \cY } F( \hat x, y ) - \min_{x\in \cX} F(x, \hat y) \leq  \frac{ C \cdot ( D^2 + L_1^2 + L_1 L_2 D ) }{\log(1/ \rho)  }\cdot \frac{\log ^2 T + \log (1/\delta) }{    \sqrt{T} }+ \frac{C \cdot C_{\xi} L_1 D } {T} . 
		\$
	\end{theorem}
 \begin{proof}
This theorem follows from Theorem 1 in \cite{wang2017finite}, where we set $\alpha_t = \alpha /\sqrt{t} $ for all $t\geq 1$, and focus on the case where $\{ \xi_t\}_{t\geq 0}$ is geometrically $\beta$-mixing.  Under the  mixing assumption,   for any $k\geq 1$, the $k$-th mixing coefficient of $\{ \xi_t\}_{t\geq 0}$ satisfies 
$\beta (k) \leq C_{\xi}\cdot \rho^k$. Then, for any $\delta, \eta \in (0,1)$, Theorem 1 in \cite{wang2017finite} implies 
\#\label{eq:thm1_bound}
	\max_{y \in \cY } F( \hat x, y ) - \min_{x\in \cX} F(x, \hat y)   & \leq   \biggl( \sum_{t=1}^T \alpha_t \biggr )^{-1} \Bigg( A_0  + A_1\cdot  \eta \cdot\sum_{t=1}^T \alpha_t + A_2 \sum_{t=1}^T \alpha_t^2     + \\
	& \qquad  \qquad 16 D  L_1 \cdot \biggl  \{  2 \tau(\eta) \cdot \log [  \tau(\eta) / \delta ] \cdot   \bigg[  \sum_{t=1}^T \alpha_t^2 + \tau(\eta) \cdot \alpha_1  \biggr ]\biggr \} ^{1/2}\Biggr ) , \notag 
\#
where we define $\tau(\eta) =  \log (\eta / C_{\xi} ) / \log (\rho)$ and denote  
$$
A_0 = D^2 +12 D \cdot \alpha_1 \cdot \tau(\eta) \qquad A_1 = 4 L_1 D \qquad A_2 = 10 L_1^2 + ( 24 L_1^2 + 8L_1 L_2 D ) \cdot \tau(\eta).
$$
Now we set $\alpha_t = \alpha /\sqrt{t}$ and $\eta = C_{\xi } /T$ in \eqref{eq:thm1_bound}, which implies that $\tau(\eta) =  \log  T/ \log (1/\rho) $.   Moreover, 
 note  that for all $T \geq 1$, we have 
$
2 \sqrt{T+1} -2 \leq \sum_{t=1}^T 1/ \sqrt{t}  \leq 2 \sqrt{T} -1  
$ 
 and $\sum_{t=1}^T  1/ t \leq \log T + 1$. The last term on the right-hand side of \eqref{eq:thm1_bound} can be upper bounded by 
 \#\label{eq:term1}
 & 16 D  L_1 \cdot \bigl  \{  2\log  T/ \log (1/\rho)   \cdot \log [  \tau(\eta) / \delta ] \cdot      \big[ \log T +1 + \alpha \cdot    \log  T/  \log (1/\rho)  \bigr ]\bigr \} ^{1/2}  \notag \\
 & \qquad \leq 16 D  L_1 \cdot \bigl  \{  2\log  T/ \log (1/\rho)   \cdot   [  \log \log T + \log (1/ \delta) ] \cdot      \big[ \log T +1 + \alpha \cdot    \log  T/  \log (1/\rho)  \bigr ]\bigr \} ^{1/2} \notag \\
 & \qquad \leq C \cdot D L_1 \cdot \log T / \log (1/\rho) \cdot \sqrt{\log \log T + \log (1/ \delta)},
 \#
 where $C$ is an absolute constant. Moreover, for the first three terms, we have 
  \#
  & A_0 = D^2 + 12 D \cdot \alpha \cdot \log  T/ \log (1/\rho) \leq C \cdot D ^2  \log  T/ \log (1/\rho) ,   \quad A_1 \cdot \eta \leq C \cdot C_{\xi} L_1 D / T, \label{eq:term2} \\
 & \qquad \qquad\qquad  A_2 \cdot \sum_{t=1}^T \alpha_t^2\leq \bigl [  10 L_1^2 + ( 24 L_1^2 + 8L_1 L_2 D ) \cdot \log  T/  \log (1/\rho) \bigr ]  \cdot (\log  T + 1) \notag \\
  & \qquad \qquad \qquad \qquad \qquad ~~\leq  C \cdot [ L_1^2 + L_1 L_2 D ]\cdot \log ^2 T/  \log (1/\rho) . \label{eq:term3}
  \#
  Thus, combining \eqref{eq:thm1_bound}, \eqref{eq:term1}, \eqref{eq:term2}, and \eqref{eq:term3},  we obtain that 
  \$
  	\max_{y \in \cY } F( \hat x, y ) - \min_{x\in \cX} F(x, \hat y) \leq C\cdot \bigl [  ( D^2 + L_1^2 + L_1 L_2 D ) / \log(1/ \rho) \cdot  \log T \cdot \log (T/ \delta )   / \sqrt{T} +C_{\xi} L_1 D / T \bigr ] ,
  \$ 
  which concludes the proof of Theorem \ref{thm:saddle_gap}.
\end{proof}

In order to apply  Theorem \ref{thm:saddle_gap} to the minimax optimization in \eqref{eq:opt_trunc}, we only need to specify parameters $C_{\xi}$,  $D$, $L_1$, and $L_2$. 
First, for any $\rho \in ( \rho(A- BK), 1)$, by Lemma \ref{lemma:lds_mix}, we can set 
\#\label{eq:set_paramc}
C_{\xi} & = C_{\rho, L } \cdot \bigl [ \tr( \tilde \Sigma_K ) + (d+k ) \cdot  ( 1-\rho)^2] ^{1/2 } \notag \\
&  \leq 2 C_{\rho, L } \cdot \sqrt{d+k} \cdot  \bigl \{    \bigl [   \sigma ^2 +   ( 1+ \| K \|_{\fro}^2  ) \cdot  \| \Sigma_K \|    \bigr ] ^{1/2}+ (1 - \rho)^{-1} \bigr \} .
\#
Moreover, by the definitions of $\cX_{\Theta}$ and $\cX_{\Omega}$ in \eqref{eq:projectsets} and \eqref{eq:projectsets2}, respectively, we can set $D$ by 
\# \label{eq:set_paramd}
D^2 = 2 [J(K_0) ]^2 +\tilde R_{\Theta}^2 +  ( 1+ \| K \|_{\fro}^2  )^4 \cdot \tilde R_{\Omega}^2.
\#
Moreover, by \eqref{eq:grad_norm_bound1},   \eqref{eq:grad_norm_bound2}, and the form of $\nabla^2 G(\theta, \omega; \phi, \phi')$, we have 
\#\label{eq:set_paraml}
L_1 \leq  16  C_1  ^2 \cdot  \log^2  T  \cdot \bigl [ \sigma^2 + ( 1+ \| K \|_{\fro}^2  )    \cdot \|  \Sigma_K  \|        \bigr ]^2   \cdot  \bigl [( 1+ \| K \|_{\fro}^2  ) ^2 \cdot   \tilde R_{\Omega} + \tilde R_{\Theta} \bigr ]  , \qquad L_2 = 1. 
\#
Combining Theorem \ref{thm:saddle_gap}, \eqref{eq:set_paramc}, \eqref{eq:set_paramd}, and  \eqref{eq:set_paraml}, we 
to obtain an upper bound for the primal-dual gap in \eqref{eq:def_gap}. Specifically, for any $\rho \in ( \rho( A - BK), 1)$ and any $\delta \in (0,1)$, with probability at least $1- \delta$,  the primal-dual gap of the optimization problem in \eqref{eq:opt_trunc} is bounded by 
\#\label{eq:trash23}
  &    C\cdot \log ^4 T \cdot  \bigl [ \sigma^2 + ( 1+ \| K \|_{\fro}^2  )    \cdot \|  \Sigma_K  \|        \bigr ]^4   \cdot \bigl [  ( 1+ \| K \|_{\fro}^2  )^2 
 \cdot \tilde  R_{\Omega} + \tilde R_{\Theta}  \bigr ]^2  \notag \\
 & \qquad \qquad   \cdot \biggl ( \frac{  \log ^2 T  + \log (1/ \delta)   }{ \log (1/ \rho)\cdot \sqrt{T}} + \frac{ \sqrt{d +k }   } {( 1- \rho)\cdot T }  \biggr ) . 
\# 
where $C > 0$ is an absolute constant.   
Besides, we note that $\sigma$ is a constant and that $\| \Sigma_K \| \geq \sigma_{\min} (\Psi ) > 0$.
 Finally, recall that, when event $\cE$ holds, the primal-dual gap is equal to $\max_{\omega \in \cX_{\Omega}}  \tilde F( \hat  \vartheta, \omega )  - \min_{ \vartheta \in \cX_{\Theta}} \tilde F (  \vartheta, \hat \omega )$.  Combining \eqref{eq:gap_diff},  \eqref{eq:trash23} with $\delta = T^{-5}$, and the fact that $\PP(\cE) \geq 1 - 2T^{-5}$,  we conclude that
 \#\label{eq:final_gap}
   \texttt{Gap} (\hat  \vartheta, \hat \omega )& \leq C\cdot \log ^4 T \cdot  ( 1+ \| K \|_{\fro}^2  ) ^4 \cdot  \| \Sigma_K \|^4 \cdot \bigl [  ( 1+ \| K \|_{\fro}^2  )^2 \cdot   \tilde  R_{\Omega} + \tilde R_{\Theta} \bigr ] ^2 \notag \\
   & \qquad \qquad \cdot \biggl ( \frac{  \log ^2 T  + \log (T^5)   }{ \log (1/ \rho)\cdot \sqrt{T}} + \frac{ \sqrt{d +k }   } {( 1- \rho)\cdot T }  \biggr )  + \frac{2}{T} \notag \\
 &  \leq   C\cdot   ( 1+ \| K \|_{\fro}^2  ) ^4 \cdot  \| \Sigma_K \|^4 \cdot\bigl [  ( 1+ \| K \|_{\fro}^2  )^2 \cdot   \tilde  R_{\Omega} + \tilde R_{\Theta} \bigr ] ^2\cdot   \frac{  \log ^6 T     }{ ( 1- \rho)\cdot  \sqrt{T}}  
 \#
 holds with probability at least   $1 -  3T^{-5} \geq 1- T^{-4}$,  
 where   in the second inequality we use  the fact that $1 -  1/x < \log x < x + 1$ holds for all $x > 0$, which implies that    $1/ \log (1/ \rho) \leq 1 / (1- \rho)  $. This further implies that  the first term on the right-hand side of the first inequality  dominates the second term.  The upper bound of   $ \texttt{Gap} (\hat  \vartheta, \hat \omega )$   in \eqref{eq:final_gap} concludes  the  last step of our proof.
 Finally, combining  \eqref{eq:step2} and   \eqref{eq:final_gap}, we complete the proof of Theorem \ref{thm:pe}.
	\end{proof}

   \subsection{Proof of Theorem \ref{thm:ac}}\label{proof:thm_ac}
   \begin{proof}
   Our proof of the global convergence can be decomposed into two steps. In the first step, similar to the analysis in \cite{fazel2018global}, we study the geometry of the average return $J(K)$, as a function of $K$. Specifically, we show that $J(K)$ is gradient dominated \citep{polyak1963gradient}. Note that we study the ergodic setting with system noise and  stochastic policies.  In contrast,   \cite{fazel2018global} study the case where both the transition and the policy are deterministic. Thus, their analysis of the geometry of $J(K)$ cannot be directly applied to our problem.  Motivated by their analysis, we follow the similar approach to  with modifications for our setting. In addition, in the second step, we utilize the geometry of $J(K)$ to show the global convergence of the actor-critic algorithm. Specifically, combining Theorem \ref{thm:pe}, we show that, with high probability, Algorithm \ref{algo:ac} constructs a sequence of policies that converges linearly to the optimal policy $\pi_{K^*}$.

   {\noindent \textbf{Step 1.}}  
   As shown in \eqref{eq:cost_K2} in Proposition \ref{prop:pg}, we can write $J(K)$ as 
\$
J(K) = \tr(P_K \Psi_{\sigma}) + \sigma^2 \cdot \tr(R) = \EE _{x\in N(0, \Psi_{\sigma})} \bigl (x^\top P_K x \bigr ) + \sigma^2 \cdot \tr(R).
\$
  In the following lemma, for two policies $\pi_K$ and $\pi_{K'}$, we bound the difference between $x ^\top P_K x$ and $x^\top P_{K'} x$. Then, taking expectation with respect to $x\in N(0, \Psi_{\sigma})$ yields the difference between $J(K)$ and $J(K')$.
  
\begin{lemma} \label{lemma:cost_diff}
Let $K$ and $K'$ be two stable policies such that  both $\rho(A- BK)$ and $\rho(A - BK')$ are smaller than one. For any $x\in \RR^d$, let $\{ x_t'\}_{t\geq 0} \subseteq \RR^d$ be the sequence of states satisfying $x_0' = x$ and 
$
x_{t+1} ' = (A- BK')  x_t ' $
for all $t\geq 0$.  Then it holds that 
\$ 
x ^\top P_{K'} x -  x^\top P_K x  = \sum_{t\geq 0} A_{K, K'} (x_t' ) , 
\$
where the function $A_{K,K'} \colon \RR^d \rightarrow \RR^d$ is defined as 
 $$A_{K, K'} (x) =2x^\top(K'-K)^\top E_K x + x^\top (K'-K)^\top (R + B^\top P_KB) (K'-K) x .$$ 
\end{lemma}

\begin{proof}
Note that  both $P_K$ and $P_{K'}$ satisfy the Bellman equation specified in \eqref{eq:bellman}.  
Moreover, using the operator $\cT_K^\top $ defined in \eqref{eq:define_operators}, we have  
$P_{K'} = \cT_{K'} ^\top (Q + K'^\top R K')$, which is equivalent to
\#\label{eq:cost11}
x ^\top P_{K'} x  = \sum_{t\geq 0} x^\top  [ ( A - BK')^t ] ^\top\big ( Q + {K'}^\top R K' \big) [ ( A - BK')^t]  x .
\#
By the construction in Lemma \eqref{lemma:cost_diff},  for all $t\geq 0$, we have $( A - BK')^t  x = x_t'$.  
  Thus, by \eqref{eq:cost11} we have 
\#\label{eq:cost12}
x ^\top P_{K'} x   =  \sum_{t\geq 0} {x_t' }^\top \big (  Q + {K'}^\top R K' \big) x_t' =  \sum_{t\geq 0} \big  ({x_t'}^\top Q x_t ' + {u_t'} ^\top R u_t' \big),
\#
where we define $u_t' = -K ' x_t'$ for all $t\geq 0$. Thus, by \eqref{eq:cost12}, we have the following telescoping sum:
\#\label{eq:cost13}
x ^\top P_{K'} x  -  x^\top P_K x & = \sum_{t\geq 0} \bigl [  ({x_t'}^\top Q x_t ' + {u_t'} ^\top R u_t' \big)   +{x_t' }^\top P_K x_t' - {x_t'}^\top P_K {x_t'}  \bigr ]  - {x_0'}^\top P _K {x_0'} \notag \\
& =  \sum_{t\geq 0} \bigl [  ({x_t'}^\top Q x_t ' + {u_t'} ^\top R u_t' \big)   +{x_{t+1}'  }^\top P_K x_{t+1}' - {x_t'}^\top P_K {x_t'}  \bigr ]  . 
\#
Thus, in \eqref{eq:cost13} we write $x ^\top P_{K'} x  -  x^\top P_K x $ as a summation where each term can be written as a quadratic function of $x_t$.
To further simplify \eqref{eq:cost13}, for any $x \in \RR^d$, we have 
\#\label{eq:cost14}
 & x ^\top Q x +  (-K'x)^\top R ( -K'x ) + [ (A -BK')x ]^\top P_K [ (A-BK') x] - x ^\top P_K x   \\
 &\qquad =   x^\top \bigl[ Q+(K'-K+K)^\top R (K'-K+K) \bigr ] x  + \notag  \\
& \qquad \qquad \qquad  x^\top \bigl [ A-BK- B(K'-K)\bigr ] ^\top P_K \bigl [ A -BK-B(K'-K)\bigr ]x - x ^\top P_K x \notag  \\
 & \qquad =   2x^\top(K'-K)^\top \bigl[  (R+ B^\top P_K B) K - B^\top P_K A\bigr ] x + x^\top (K'-K)^\top (R + B^\top P_KB) (K'-K) x . \notag \\ 
 & \qquad = 2x^\top(K'-K)^\top  E_K  x + x^\top (K'-K)^\top (R + B^\top P_KB) (K'-K) x ,\notag 
\#
where $E_K = ( R   +  B^\top P_K B ) K - B^\top P_K A $.
Finally, combining \eqref{eq:cost13} and  \eqref{eq:cost14}, we complete the proof of this lemma.
\end{proof}

In the following lemma, we utilize Lemma \ref{lemma:cost_diff} to show that $J(K)$ is gradient dominated.
 
\begin{lemma}  
[Gradient domination of $J(K)$] \label{lemma:domination} Let $K^*$ be an optimal policy. Suppose $K$ has
finite cost. Then, it holds that
\#\label{eq:lemmadom}
\sigma_{\min}(\Psi) \cdot \|R + B^\top P_KB\| ^{-1} \cdot  \tr (E_K^\top E_K) &  \leq  J(K)-J(K^*)  \notag \\
&    \leq 
 1/ \sigma_{\min} (R)  \cdot  \| \Sigma_{K^*}\| \cdot   \tr (E_K^\top E_K).
\#
  \end{lemma}

\begin{proof}
For the upper bound in  \eqref{eq:lemmadom}, bu \eqref{eq:cost_K2} we obtain that 
\#\label{eq:graddom1}
J(K) - J(K^*) = \tr [ ( P_K - P_K^*) \Psi _{\sigma}] = \EE_{x\sim N(0, \Psi_{\sigma})} \bigl [  x^\top( P_K - P_K^*)  x\bigr ] , 
\#
where $ \Psi _{\sigma} = \Psi + \sigma^2 BB^\top$ does not involve $K$ or $K^*$.   Applying Lemma \ref{lemma:cost_diff} to \eqref{eq:graddom1} with $K' = K^*$,  we have 
\#\label{eq:graddom2}
J(K) - J(K^*)  = - \EE_{x_0^* \sim N(0, \Psi_{\sigma})} \biggl [  \sum_{t\geq 0} A_{K, K^*} (x_t^*) \biggr ],
\#
where we define 
$x_t^* = (A - BK^*)^t x_0^*$ for all $t\geq 0$.  Besides, by direct computation, we have 
\#\label{eq:graddom3}
& \EE _{x_0^*\sim N(0, \Psi_{\sigma})}  \biggl [ \sum_{t\geq 0} x^*_t(x^*_t)^\top \biggr ]     \notag \\
& \qquad = \EE _{x\sim N(0, \Psi_{\sigma})} \bigg \{  \sum_{t\geq 0}  (A - BK^*)^t x x^\top  [(A - BK^*)^t]^\top  \bigg \}  = \cT_{K^*} (\Psi_{\sigma}) = \Sigma_{K^*},
\#
where the  operator $\cT_K$ is defined in \eqref{eq:define_operators}. 

Meanwhile, by the definition of $A_{K, K'}$, for any $x\in \RR^d$,  by completing the squares we have 
\#\label{eq:quad_lower_bound}
& A_{K, K'}(x)  = 2x^\top(K'-K)^\top E_K x + x^\top (K'-K)^\top (R + B^\top P_KB) (K'-K) x  \notag \\
& \qquad =  \tr \Bigl\{ xx^\top\bigl [ K'-K+(R + B^\top P_KB)^{-1}E_K \bigr ]^\top (R + B^\top P_KB)
    \bigl[ K'-K+(R + B^\top P_KB)^{-1}E_K \bigr ] \Bigr \} \notag \\
    & \qquad \qquad \qquad  -\tr\bigl [ xx^\top E_K^\top(R + B^\top P_KB)^{-1}E_K\bigr ] \notag \\
    & \qquad \geq  -\tr\bigl [ xx^\top E_K^\top(R + B^\top P_KB)^{-1}E_K\bigr ],
\#
where the equality is attained by 
 $K'=K-(R + B^\top P_KB)^{-1}E_K$.

Thus, combining \eqref{eq:graddom2}, \eqref{eq:graddom3}, and  \eqref{eq:quad_lower_bound}, we obtain that 
\#\label{eq:graddom4}
J(K) - J(K^*)  & \leq   \tr\bigl [ \Sigma_{K^*} E_K^\top(R + B^\top P_KB)^{-1}E_K\bigr ]   \leq \| \Sigma_{K^*}\| \cdot \tr\bigl [ \Sigma_{K^*} E_K^\top(R + B^\top P_KB)^{-1}E_K\bigr ]  \notag \\
& \leq   \| \Sigma_{K^*}\|  \cdot \| (R + B^\top P_KB)^{-1} \| \cdot \tr (E_K^\top E_K) .
\#
Notice that $R + B^\top P_KB \succeq R$ implies  $(R + B^\top P_KB)^{-1} \preceq R^{-1} $. Therefore, by \eqref{eq:graddom4} we obtain that 
$
J(K) - J(K^*)   \leq   1/ \sigma_{\min} (R) \cdot  \| \Sigma_{K^*}\| \cdot  \tr (E_K^\top E_K),
 $
 which establishes the upper bound in  \eqref{eq:lemmadom}.
 
 Furthermore, for the lower bound, since $K'=K-(R + B^\top P_KB)^{-1}E_K$ attains the lower bound in \eqref{eq:quad_lower_bound} and $K^*$ is the optimal policy, similar to \eqref{eq:graddom2} and \eqref{eq:graddom3}, we have 
\$
 & J(K) - J(K^*) \geq J (K) - J(K')  = - \EE_{x_0^* \sim N(0, \Psi_{\sigma})} \biggl [  \sum_{t\geq 0} A_{K, K'} (x_t') \biggr ]  
  \\
 & \qquad = \tr \bigl [   \Sigma_{K'} E_K^\top (R + B^\top P_KB)^{-1} E_K \bigr ] \geq  \sigma_{\min}(\Psi) \cdot \|R + B^\top P_KB\|  ^{-1} \cdot  \tr(E_K^\top E_K),
\$
where in the first equality we define $x_t' = (A - BK') ^t $ for all $t\geq 0$, and the last inequality follows from the fact that $\Sigma_{K'}  \succeq  \Psi\succeq  \sigma_{\min}(\Psi) \cdot I_d$.  Therefore, we conclude the proof of Lemma \ref{lemma:domination}.
\end{proof}
 
 Notice that $K = K^*$ achieves the minimum of $J(K)$.  Lemma \ref{lemma:domination} implies that   
 \$
 J(K)-J(K^*)   
  \leq  \lambda 
  \cdot  \la E_K, E_K\ra,
 \$
 where $\lambda = 1/ \sigma_{\min} (R)  \cdot  \| \Sigma_{K^*}\|$. That is,  the difference of the objective can be bounded by the norm of the natural gradient. Therefore, updating the policy parameter $K$ in the direction of natural gradient $E_K$ yields decreases the objective value. Therefore, we conclude the first step.

    {\noindent \textbf{Step 2.}}   In the second part of the proof, equipped with Lemma \ref{lemma:domination}, we establish the global convergence of the natural actor-critic algorithm. Recall that we assume that the initial policy $\pi_{K_0}$ is stable, which implies that $J(K_0)$ is finite. 
    Moreover, according to Algorithm \ref{algo:ac}, the policy parameters are updated via 
    \#\label{eq:npg_update}
    K_{t+1}    = K_t - \gamma  \cdot \hat E_{K_t} , \qquad \hat E_{K_t} = \hat \Theta_{t }^{22} K_{t} - \hat \Theta_t ^{21},
    \#
    where $\hat \Theta_t$ is the estimator of $\Theta_{K_t}$ returned by Algorithm \ref{algo:gtd}.  
    
    We use mathematical induction to show that $\{ J(K_t) \}_{t\geq 0}$ is a monotone decreasing sequence.  Suppose $J(K_t) \leq J(K_0)$. 
    We define $K_{t+1} '= K_t - \gamma \cdot E_{K_t}$, i.e., $ K_{t+1} '$ is obtained by a single step of natural policy gradient, starting from $K_t$. 
    In the sequel, we use $J(K_{t+1}' )$ to connect $ J(K_t)$ and $J(K_{t+1})$.
    By Lemma     \ref{lemma:cost_diff}, we have 
    \#\label{eq:npg1}
J(  K_{t+1}' ) - J(K_t) & =   \EE_{x \sim N(0, \Psi_{\sigma})} [ x ^\top ( P_{ K_{t+1}' } - P_{K_{t}}) x ] \notag \\
& = -2\gamma \cdot  \tr \bigl( \Sigma_{K_{t+1}'} \cdot E_{K_t} ^\top E_{K_t} \bigr )  + \gamma ^2 \cdot \tr \bigl [\Sigma_{K_{t+1}' } \cdot  E_{K_t} ^\top ( R + B^\top  P_{K_t}  B) E_{K_t}   \bigr ] \notag \\
& = -2\gamma  \cdot  \tr \bigl( \Sigma_{K_{t+1}'} \cdot E_{K_t} ^\top E_{K_t} \bigr )  + \gamma ^2 \cdot \| R + B^\top  P_{K_t}  B \| \cdot   \tr \bigl( \Sigma_{K_{t+1}'} \cdot E_{K_t} ^\top E_{K_t} \bigr )  .
\# 
When $\gamma $ is sufficiently small such that 
\#\label{eq:stepsize_ac}
\gamma \cdot  \bigl [ \| R \| + \sigma_{\min}^{-1} (\Psi) \cdot \| B \|^2 \cdot J(K_0)  \bigr  ] \leq 1,
\#  by triangle inequality, we have 
\#\label{eq:trash71}
\gamma \cdot \| R + B^\top  P_{K_t}  B \| \leq  \gamma \cdot \bigl [ \| R \|  + \| B\|^2 \cdot \| P_{K_t} \|  \bigr ]  \leq \gamma \cdot  \bigl [ \| R \| + \sigma_{\min}^{-1} (\Psi) \cdot \| B \|^2 \cdot J(K_0)  \bigr  ]  < 1,
\#
where the second inequality follows from Lemma \ref{lemma:bound_mats} and the induction assumption that $J(K_t) \leq J(K_0)$, and the   last inequality follows from \eqref{eq:stepsize_ac}. Thus, combining \eqref{eq:npg1} and \eqref{eq:trash71}, we have 
 \#\label{eq:npg2}
 J(K_{t+1}' ) - J(K_t)  & \leq - \gamma \cdot  \tr \bigl( \Sigma_{K_{t+1}' } \cdot E_{K_t} ^\top E_{K_t} \bigr ) \leq -  \gamma \cdot \sigma_{\min} (\Psi) \cdot \tr \bigl(   E_{K_t} ^\top E_{K_t} \bigr )   , \notag \\
 & \leq  - \gamma \cdot \sigma_{\min} (\Psi) \cdot \sigma_{\min}(R) \cdot \| \Sigma_{K^*} \|^{-1} \cdot \bigl [ J(K_t) - J(K^*) \bigr ].
 \# 
 where the third inequality follows from the fact that $\Sigma_{K_{t+1}' } \succeq \Psi   $, and the last inequality follows from  Lemma \ref{lemma:domination}.  Note that \eqref{eq:npg2} implies that $J(K_{t+1}') \leq J(K_t) \leq J(K_0)$.
 
 Furthermore,  by the difference between $J(K_{t+1})$ and $J(K_{t+1}')$ can be bounded by 
 \#\label{eq:npg3} 
 \big|  J(K_{t+1} ) - J(K_{t+1}') \big| &  = \bigl | \tr \bigl [ (P_{K_{t+1} } - P_{K_{t+1}' }) \cdot \Psi_{\sigma }\bigr ] \bigr | \leq \| \Psi_{\sigma} \|_{\fro} \cdot    \bigl  \| P_{K_{t+1} } - P_{K_{t+1}' } \bigr \|  \notag \\
 & \leq \bigl [ \| \Psi\|_{\fro} \cdot + \sigma^2 \cdot \| B \|_{\fro}^2  \bigr ] \cdot \bigl  \| P_{K_{t+1} } - P_{K_{t+1}' } \bigr \| .   
 \#
 Now we utilize the following Lemma, obtained from \cite{fazel2018global}, to construct and upper bound for  $\| P_{K_{t+1} } - P_{K_{t+1}' } \|$. 
 
\begin{lemma} [Perturbation of $P_K $]\label{lemma:SigmaK_perturbation}
	Suppose $\pi_{K'}$ is a small perturbation of $\pi_K$ in the sense that 
\#\label{eq:trash11}
	\|K'-K\|\leq \sigma_{\min}(\Psi) / 4 \cdot \| \Sigma_K \|^{-1}
	 \|B\|^{-1} \cdot ( \|A-B K\|+ 1 )  ^{-1},
\#
then we have 
\#\label{eq:trash1111}
 \| P_{K'} - P_K \| & \leq 6 \sigma_{\min}^{-1}(\Psi) \cdot   \|\Sigma_K \| \cdot \| K \| \cdot \|R\|  \notag \\
 & \qquad \qquad  \cdot \bigl ( \| K \| \cdot\| B \| \cdot \| A - BK \| + \| K \| \cdot \| B \| + 1 \bigr ) \cdot \| K - K'\| .
\#
\end{lemma}

\begin{proof}
	This lemma is a slight modification of Lemma 24 in  \cite{fazel2018global}.   Here we sketch the  proof. See \cite[Lemmas 17 and 24]{fazel2018global}   for a detailed proof.
	
	Recall that we define operator $\cT_K$ in \eqref{eq:define_operators}.
	The operator norm of $\cT_K$ is defined as $\| \cT_K \| \leq \sup_{\Omega} \|\cT_K(\Omega )\| / \| \Omega \|$, where the supremum is taken over all symmetric matrices.
	 As shown in  Lemma 17 in \cite{fazel2018global},   we have $\|\cT_K \| \leq \sigma_{\min}^{-1}(\Psi) \cdot \| \Sigma_K \| .$   Moreover, under the condition in \eqref{eq:trash11},  in the proof of Lemma 24 in \cite{fazel2018global}, it is shown that 
	 \$
	 \| P_{K'} - P_K \| & \leq 6  \|\cT_K \|  \cdot \| K \| \cdot \|R\| \cdot \bigl ( \| K \| \cdot\| B \| \cdot \| A - BK \| + \| K \| \cdot \| B \| + 1 \bigr ) \cdot \| K - K'\| .
	 \$
	 Combining this with the upper bound on $\| \cT_K \|$, we conclude the proof.
	\end{proof}
To use this lemma, we need to verify \eqref{eq:trash11}. That is, 
\#\label{eq:trash81}
4 \| K_{t+1} - K_{t+1} ' \| \cdot ( 1+ \| A - B K_{t+1}' \| ) \cdot   \| B \| \cdot \| \Sigma_{K_{t+1}'} \|  \leq \sigma_{\min}(\Psi).
\#
By the definition of $K_{t+1} $ and $K_{t+1}'$, we have 
\#\label{eq:trash82} 
  \| K_{t+1} - K_{t+1} ' \| = \gamma \cdot \| \hat E_{K_t} - E_{K_t} \| \leq \gamma \cdot  \|  \hat \Theta_t - \Theta_{K_t} \|_{\fro} \cdot ( 1+ \| K_t \| ) ,
\#
where $\hat E_{K_t}$ is defined in \eqref{eq:npg_update}. 
 Plugging \eqref{eq:trash82} into the left-hand side of  \eqref{eq:trash81}, we  obtain that 
 \#\label{eq:trash83} 
 & 4 \| K_{t+1} - K_{t+1} ' \| \cdot ( 1+ \| A - B K_{t+1}' \| ) \cdot   \| B \| \cdot \| \Sigma_{K_{t+1}'} \|  \notag  \\
 &\qquad \leq4  \gamma \cdot  \|  \hat \Theta_t - \Theta_{K_t} \|_{\fro} \cdot ( 1+ \| K_t \| )  \cdot ( 1+ \| A - B K_{t+1}' \| ) \cdot   \| B \| \cdot \| \Sigma_{K_{t+1}'} \|  .
 \#
Utilizing Lemma \eqref{lemma:bound_mats} and the fact that $J(K_{t+1} ') \leq J(K_0)$, 
we have 
\#\label{eq:some_ineq1}
\| \Sigma_{K_{t+1}'} \| \leq J(K_{t+1} ') /  \sigma_{\min}(Q)  \leq J(K_0) /  \sigma_{\min}(Q)  .
\#
 In addition, by triangle inequality, we have 
 \#\label{eq:some_ineq2}
 & \| A - BK_{t+1}' \| \leq  \| A - B K_{t} \| + \gamma \cdot  \| B \| \cdot \| E_{K_t} \| \notag \\
 &\qquad  \leq  \| A - B K_{t} \| + \gamma \cdot  \| B \| \cdot \| \Theta_{K_t} \| \cdot ( 1+ \| K_t\| ).
 \#
 By the definition of $\Theta_K$ in \eqref{eq:ThetaK}, we have 
 \#\label{eq:some_ineq3}
  & \| \Theta_{K_t} \|  \leq \| Q \| + \| R \| +( \| A \| _{\fro}+ \| B \| _{\fro})^2 \cdot \|P_{K_t } \| \notag \\
  & \qquad \leq   \| Q \| + \| R \| +( \| A \| _{\fro}+ \| B \| _{\fro})^2 \cdot J(K_0) / \sigma_{\min} (\Psi),
 \#
 where the last inequality follows from  Lemma \eqref{lemma:bound_mats}  and the induction assumption.
 Furthermore, by triangle inequality, it holds that 
 \#\label{eq:some_ineq4}
\| K_{t+1} \| &  \leq \| K_t\| + \gamma \cdot \| E_{K_t} \| \leq  \| K_t\| + \gamma \cdot \| \Theta_{K_t} \| \cdot ( 1+ \| K_t\| )  \notag \\
 &   \leq  \| K_t\| + \gamma \cdot \bigl [ \| Q \| + \| R \| +( \| A \| _{\fro}+ \| B \| _{\fro})^2 \cdot J(K_0) / \sigma_{\min} (\Psi) \bigr ] \cdot   ( 1+ \| K_t\| ).
 \#
 
In the sequel, we set 
\#
\label{eq:stepsize}
\gamma = \bigl [ \| R \| + \sigma_{\min}^{-1} (\Psi) \cdot \| B \|^2 \cdot J(K_0)  \bigr  ] ^{-1}.
\#
Note that we assume  that $\| Q \|$, $\| R \|$, $\|A \|$, $\|B \|$, $\sigma_{\min}(Q )$, $\sigma_{\min}(R)$ are all constants. Combining   \eqref{eq:trash83}, \eqref{eq:some_ineq1}, \eqref{eq:some_ineq2},  and \eqref{eq:some_ineq3}, we conclude that there exists a polynomial $\Upsilon_1(\cdot , \cdot )$ such that 
     \#\label{eq:one_side}
     4 \| K_{t+1} - K_{t+1} ' \| \cdot ( 1+ \| A - B K_{t+1}' \| ) \cdot   \| B \| \cdot \| \Sigma_{K_{t+1}'} \|  \leq \Upsilon_1\bigl[  \| K_t \|, J(K_0)  \bigr ]  \cdot \| \hat \Theta_t - \Theta_{K_t} \|_{\fro} .
     \#

Furthermore, for the right-hand side of \eqref{eq:trash1111},  combining  \eqref{eq:trash82},   \eqref{eq:trash83}, \eqref{eq:some_ineq1}, \eqref{eq:some_ineq2}, \eqref{eq:some_ineq3}, and \eqref{eq:some_ineq4}. we  conclude that there exists a polynomial $\Upsilon_2(\cdot , \cdot )$ such that 
\#\label{eq:trash84}
& \bigl [ \| \Psi\|_{\fro} \cdot + \sigma^2 \cdot \| B \|_{\fro}^2  \bigr ]  \cdot 6 \sigma_{\min}^{-1}(\Psi) \cdot   \|\Sigma_{K_{t+1}'}  \| \cdot \| K_{t+1'}  \| \cdot \|R\|  \notag \\
 & \qquad \qquad \qquad  \qquad   \cdot \bigl ( \| K_{t+1}' \| \cdot\| B \| \cdot \| A - BK_{t+1}' \| + \| K_{t+1}' \| \cdot \| B \| + 1 \bigr ) \cdot \| K _{t+1} - K_{t+1}'\|  \notag \\
 & \qquad \leq  \Upsilon_2\bigl[  \| K_t \|, J(K_0)  \bigr ] \cdot \| \hat \Theta_t - \Theta_{K_t} \|_{\fro}.
\#
 
 Meanwhile, in Theorem \ref{thm:pe} we have shown that, there exists a polynomial $\Upsilon_3(\cdot , \cdot )$  such that, for $T$ sufficiently large, Algorithm \ref{algo:gtd} with $T$ iterations returns an estimator $\hat \Theta_t$ for $\Theta_{K_t}$ such that 
 \#\label{eq:apply_pe}
 \| \hat \Theta_t - \Theta_{K_t} \|_{\fro}  \leq \frac{\Upsilon_3\bigl[  \| K_t \|, J(K_0)  \bigr ] }{\kappa_{K_t}^* \cdot \sqrt{    (1- \rho) } }\cdot \frac{\log ^3 T}{T^{1/4}} 
 \#
 holds with probability at least $1- T^{-4}$, where $\rho \in (\rho( A - BK_t) , 1)$ and $\kappa_{K_t}^*$ is specified in Lemma \ref{lemma:pe_mat1}, which depends only on $\rho$, $\sigma$, and $\sigma_{\min}(\Psi)$.  Notice that $\log ^3 T  \cdot T^{-1/4} \leq T^{-1/5}$ for $T$ sufficiently large. Therefore, in the GTD algorithm for estimating $\Theta_{K_t}$, we set the number of iterations $T_t$ sufficiently large such that 
 \#
&  \Upsilon_1\bigl[  \| K_t \|, J(K_0)  \bigr ] \cdot  \Upsilon_3\bigl[  \| K_t \|, J(K_0)  \bigr ] \cdot  {\kappa_{K_t}^*}^{-1}    \cdot (1- \rho)  ^{-1/2} \cdot T_t^{-1/5}  \leq \sigma_{\min}(\Psi ),\notag  \\
&  \Upsilon_2\bigl[  \| K_t \|, J(K_0)  \bigr ] \cdot \Upsilon_3\bigl[  \| K_t \|, J(K_0)  \bigr ]  \cdot {\kappa_{K_t}^*}^{-1}     \cdot (1- \rho)  ^{-1/2} \cdot T_t^{-1/5} \notag \\
& \qquad \qquad \qquad   \leq \epsilon /2 \cdot  \sigma_{\min} (\Psi) \cdot \sigma_{\min}(R) \cdot \| \Sigma_{K^*} \|^{-1} \label{eq:set_innerT}
 \#
 hold simultaneously.  For such a $T_t$, combining \eqref{eq:one_side} and \eqref{eq:apply_pe}, we conclude that  \eqref{eq:trash81} holds.  Lemma \ref{lemma:SigmaK_perturbation}  implies that \eqref{eq:trash1111} is true. 
 Combining \eqref{eq:npg3}, \eqref{eq:trash1111}, \eqref{eq:trash84}, and  \eqref{eq:apply_pe}, we conclude that 
 \#\label{eq:obj_diff_final}
 \big|  J(K_{t+1} ) - J(K_{t+1}') \big| \leq  \epsilon /2 \cdot  \sigma_{\min} (\Psi) \cdot \sigma_{\min}(R) \cdot \| \Sigma_{K^*} \|^{-1}
 \#
 holds with probability at least $1 - T_t^{-4}$.  
 Thus, when $J(K_t) - J(K^*)  > \epsilon$, combining  \eqref{eq:npg2} and \eqref{eq:obj_diff_final} we have 
 \$
 J(K_{t+1} ) - J( K_t) \leq - \epsilon /2  \cdot \gamma \sigma_{\min} (\Psi) \cdot \sigma_{\min}(R) \cdot \| \Sigma_{K^*} \|^{-1}  < 0.
 \$
 Therefore, we have shown that, as long as $J(K_t) - J(K^*) \geq \epsilon$,  $J(K_{t+1} ) <J( K_t) $ holds with probability at least $1- T_t^{-1/4}$. 
 
 Meanwhile, \eqref{eq:npg2} implies that, 
 \$
 J(K_{t+1}' ) - J(K^* ) \leq \bigl [ 1 - \gamma \cdot \sigma_{\min} (\Psi) \cdot \sigma_{\min}(R) \cdot \| \Sigma_{K^*} \|^{-1}  \bigr ] \cdot \bigl [ J(K_{t } ) - J(K^* )   \bigr ]
 \$
 By \eqref{eq:obj_diff_final}, when $J(K_t) - J(K^*) \geq \epsilon$, with probability $1 - T_t^{-4}$, we have 
 \$
 J(K_{t+1}) - J(K^* )  \leq \bigl [ 1 - \gamma / 2\cdot \sigma_{\min} (\Psi) \cdot \sigma_{\min}(R) \cdot \| \Sigma_{K^*} \|^{-1}  \bigr ] \cdot \bigl [ J(K_{t } ) - J(K^* )   \bigr ],
 \$
 which shows that, in terms of the policy parameter, natural actor-critic algorithm converges linearly. Specifically, with 
 \#\label{eq:total_num_iter}
 N \geq  2 \| \Sigma_{K^*} \|  / \gamma \cdot \sigma_{\min}^{-1} (\Psi) \cdot \sigma_{\min}^{-1} ( R)   \cdot \log \bigl \{ 2 [   J(K_0) - J(K^*) ]  / \epsilon \bigr \} 
 \#  policy updates, we  have $J(K_N) - J(K^*) \leq \epsilon$ with high-probability,   
 where $\gamma $ is specified in \eqref{eq:stepsize}.
 
 Finally, it remains to determine $T_t $ for all $t \in [N]$. Notice that $T_t$ satisfies the two inequalities in \eqref{eq:set_innerT}. Thus, we set 
 \$T_t  \geq \Upsilon_4[ \| K_t \|, J(K_0) ] \cdot {\kappa_{K_t}^*}^{-5} \cdot  (\Xi_{K_t}) \cdot \bigl[  1- \rho( A - BK_t)  \bigr ]^{-5/2} \cdot \epsilon^{-5} \$ for some polynomial function $\Upsilon_4(\cdot , \cdot)$. With such a $T_t$, the fail probability $T_t^{-4} \leq \epsilon^{-20}$. Notice that the total number of iterations depends on $\epsilon $ only through $\log (1/ \epsilon)$. Thus, the total fail probability can be bounded by $\epsilon^{10}$. Therefore, we conclude the proof.
  \end{proof}

\section{Conclusion}
For  linear quadratic regulator  with ergodic cost, we 
 propose an online natural actor-critic algorithm with GTD  policy evaluation updates. The proposed algorithm is shown to find the optimal policy with  linear rate of convergence.
 Our results provide nonasymptotic theoretical justifications for actor-critic methods with function approximation, which have received tremendous empirical success recently. 
 A future direction is to extend our analysis to  linear-quadratic-Gaussian control problems \citep{kirk1970optimal}, which seems to be the simplistic model of partially observable Markov decision process.
 Another future direction is to develop model-free reinforcement learning methods for linear-quadratic dynamic games \citep{basar1999dynamic}, a classical example of  multi-agent reinforcement learning.

\clearpage

\appendix{}

\section{Off-Policy GTD Algorithm}

In Algorithm \ref{algo:off-policy} we present the details of the off-policy GTD algorithm for policy evaluation in the ergodic setting. This  algorithm can be  applied to general ergodic MDPs and is able to handle data generated from a Markov process. See \ref{sec:offac} for the derivation of this algorithm.

\begin{algorithm} [ht]
	\caption{Off-Policy Gradient-Based Temporal-Difference Algorithm for Policy Evaluation} 
	\label{algo:off-policy} 
	\begin{algorithmic} 
		\STATE{{\textbf{Input:}} Policy $\pi_K$, number of iterations $T$,  and  stepsizes   $\{ \alpha_t\}_{t\in [T] }$, the behavior policy $\pi_b$ and its stationary distribution $\rho_b$.}  
		\STATE{{\textbf{Output:} Estimators $\hat J$ and  $\hat \Theta $ of $J(K)$ in \eqref{eq:cost_K2} and $\Theta_K$  in \eqref{eq:ThetaK}, respectively.}}
	 	\STATE{Initialize the primal  and dual variables by $\vartheta_0 \in \cX_{\Theta} $ and $\omega_0\in \cX_{\Omega} $, respectively.}
	 	\STATE{Sample the  initial state $x_0\in \RR^d$ from the stationary distribution  $\rho_b$. Take action $u_0\sim \pi_b(\cdot \given x_0)$ and obtain the reward $c_0$ and the next state $x_1$.}
		\FOR{$t= 1 ,  2, \ldots,  T $}
		\STATE{Take action $u_{t}$ according to policy $\pi_K$, observe the reward $c_t$ and the next state $x_{t+1}$.}
		\STATE{Compute the TD-error $ \delta_t = \vartheta_{t-1}^1  - c_{t-1} +  [ \phi(x_{t-1}, u_{t-1}  ) - \tau_K(x_t, u_t) \cdot \phi(x_{t}, u_{t} )] ^\top \vartheta_{t-1} ^2$. }
		\STATE{Update the primal variable $\vartheta$ by 
		\$   \vartheta_{t} ^1 & = \vartheta_{t-1} ^1 -  \alpha _t   \cdot [ \omega_{t-1} ^1 +   \phi(x_{t-1}, u_{t-1}) ^\top \omega_{t-1} ^2] , \\
		 \vartheta_{t} ^2& = \vartheta_{t-1}^2 -   \alpha _t   \cdot [  \phi(x_{t-1}, u_{t-1} )- \tau_K (x_t, u_t) \cdot \phi(x_t, u_t )]  \cdot [  \phi(x_{t-1}, u_{t-1}) ^\top \omega_{t-1}^2 + \omega_{t-1}^1] . 
		\$} 
		\STATE{Update  the dual variable $\omega$ by   
		\$  \omega_{t}^1  &= ( 1-\alpha_t ) \cdot  \omega_t ^1 + \alpha_t \cdot \bigl\{ \vartheta_{t-1}^1+  [  \phi(x_{t-1}, u_{t-1} )- \tau_K (x_t, u_t) \cdot \phi(x_t, u_t )]  ^\top \vartheta_{t-1}^2  - c_{t-1}\bigr \}, \\
		  \omega_{t}^2 & = ( 1 - \alpha_t ) \cdot \omega_t ^2 + \alpha_t \cdot \delta _t \cdot \phi(x_{t-1}) .
		\$} 
		\STATE{Project $\vartheta_t $ and $\omega_t$ to $\cX_{\Theta}$ and $\cX_{\Omega}$, respectively.} 
		\ENDFOR
		\STATE{Define $\hat \vartheta =(\hat \vartheta^1, \hat \vartheta^2)  =   (  \sum_{t=1}^{T} \alpha_t \cdot \vartheta_t )/ (  \sum_{t=1}^{T} \alpha _t ) $ and $\hat \omega =  (  \sum_{t=1}^{T} \alpha_t \cdot \omega_t ) / (  \sum_{t=1}^{T} \alpha _t ) $}.
		 \STATE{Return $\hat \vartheta^1$ and $\hat \Theta = \smat(\hat \vartheta^2)$ as the estimators  of  $J(K)$ and  $\Theta_K$, respectively.}
	\end{algorithmic}
\end{algorithm}

\section{Proofs of the Auxiliary Results}
In this section, we provides the proofs for Proposition \ref{prop:pg} and Lemma  \ref{lemma:pe_mat1}. 

\subsection{Proof of Proposition \ref{prop:pg}} \label{proof:prop:pg}
\begin{proof}
	We first establish \eqref{eq:cost_K2}. 
	Note that under $\pi_K$, we can write $u_t$  as $ -K x_t + \sigma \cdot \eta_t$, where $ \eta_t \sim N(0,I_d)$. 
	This implies that, for all $\geq 0$, we have 
	\#\label{eq:11}
	\EE [  c(x_t, u_t ) \given  x_t ]&  =   x_t^\top Q x_t + \EE _{\eta_t\sim N(0, I_d)} [ ( -K x_t + \sigma \cdot \eta_t) ^\top R( -K x_t + \sigma \cdot \eta_t)   ]  \notag \\
	& = x_t^\top ( Q + K^\top R K ) x_t + \sigma^2 \cdot \tr(R).
	\#
	Thus,  
	combining \eqref{eq:11} and  the definition of $J(K)$ in \eqref{eq:cost}, we have 
	\#\label{eq:new_cost1}
	J(K) & =  \lim_{T \rightarrow \infty }  \EE \biggl  \{ \frac{1}{T}  \sum_{t\geq 0} ^T\EE [  c(x_t, u_t ) \given  x_t] \biggr \}  =  \lim_{T \rightarrow \infty }  \EE \biggl  \{ \frac{1}{T}  \sum_{t\geq 0} ^T [ x_t ^\top ( Q + K^\top R K ) x_t + \sigma^2 \cdot \tr(R) ]  \biggr \}  \notag \\
	& = \EE_{x \sim\rho_K } [ x ^\top ( Q + K^\top R K )  x ]   + \sigma^2 \cdot \tr(R)  = \tr \bigl [  ( Q+ K^\top R K ) \Sigma_K \bigr ] + \sigma^2 \cdot \tr(R), 
	\#
	where the third inequality in \eqref{eq:new_cost1} holds because the limiting distribution of $\{ x_t\}_{t\geq 0}$ is $\rho_K$.
	
	It remains to establish the second equality in \eqref{eq:cost_K2}.  To this end,  for $K \in \RR^{k\times d} $ such that $\rho(A - BK ) < 1$, we define operators we define $\cT_K $ and $\cT_K^\top$ by
	\#\label{eq:define_operators}
	\cT_K(\Omega ) = \sum_{t\geq 0} ( A - BK)^t \Omega \bigl[ ( A - BK)^t \bigr ]^\top, \qquad \cT_K^\top (\Omega ) = \sum_{t\geq 0} \bigl[ ( A - BK)^t \bigr ]^\top \Omega ( A - BK)^t, 
	\#
	where $\Omega \in \RR^{d\times d}$
	is positive definite. 
	By definition, $\cT_K(\Omega) $ and   $\cT_K^\top (\Omega)$ satisfy   Lyapunov equations
	\#
	\cT_K(\Omega)   & = \Omega + ( A- BK)  \cT_K(\Omega)   ( A- BK)^\top, \label{eq:new_bellman}\\
	\cT_K^\top (\Omega)  & = \Omega + ( A- BK) ^\top  \cT_K^\top (\Omega)   ( A- BK) ,\label{eq:new_bellman2}
	\#
	respectively. 
	Moreover, for any positive definite matrices $\Omega_1, \Omega_2$,  since $\rho(A - BK) < 1$,  we have 
	\#
	\tr [\Omega_1 \cdot  \cT_K (\Omega_2)  ] &  = \sum_{t\geq 0}   \tr \bigl \{   \Omega_1     ( A - BK) ^t \Omega_2 [ (A - BK)^t ] ^\top \bigr \} \notag \\
	&  =  \sum_{t\geq 0}   \tr \bigl \{ [ (A - BK)^t ] ^\top    \Omega_1     ( A - BK) ^t \Omega_2  \bigr \}  = \tr [ \cT_K^\top (\Omega_1) \cdot \Omega_2]. \label{eq:relation1}
	\# 
	Meanwhile, 	 
	by combining 
	\eqref{eq:cov_equ},  \eqref{eq:bellman}, \eqref{eq:new_bellman}, and \eqref{eq:new_bellman2},  we have $\Sigma_K = \cT_K(\Psi_{\sigma}) $ and$P_K = \cT^\top _K ( Q + K^\top R K ).$ Thus, \eqref{eq:relation1} implies that    
	\$
	\tr \bigl [( Q + K^\top R K ) \cdot \Sigma_K \bigr ] = \tr \bigl [( Q + K^\top R K ) \cdot \cT_K(\Psi_{\sigma})  \bigr ]  = \tr \bigl [\cT^\top_K  ( Q + K^\top R K  ) \cdot \Psi_{\sigma}  \bigr ]  = \tr( P_K \Psi_{\sigma}). 
	\$
	Combining this equation with \eqref{eq:new_cost1}, we establish the second equation of \eqref{eq:cost_K2}. 
	
	In the following, we establish the value functions. In the 
	setting of  LQR,
	the state-value function $V_K$ is given by 
	\#\label{eq:vk_def}
	V_K(x)&= \sum_{t=0}^\infty \bigl \{ \EE [   c(x_t, u_t ) \given x_0 = x, u_t = -Kx_t + \sigma\cdot \eta _t ] - J(K) \bigr\} \notag \\
	& =  \sum_{t=0}^\infty \bigl \{ \EE[ x_t ^\top ( Q+ K^\top R K ) x_t] + \sigma^2 \cdot \tr (R) - J( K) \} .
	\#
	Combining the linear dynamics in \eqref{eq:new_dyn} and \eqref{eq:vk_def}, we see that 
	$V_K$ is a quadratic function, which is   denoted by 
	$V_k(x) = x^\top P_K x + \alpha_K$, where both $P_K $ and $\alpha_K$ depends on $K$. Note that $V_K$ satisfies   the Bellman equation 
	$$
	V_K(x) = \EE_{u\sim \pi_K} [c(x, u)]  - J(K) + \EE[ V_K(x') \given x],
	$$ 
	where $x ' $ is the next state given $(x, u)$.
	Thus, for any $x \in \RR^d$, we have 
	\$
	x^\top P_K x  =  x ( Q + K^\top R K) x  + x ^\top (A - BK)^\top P_K (A- BK) x.   
	\$
	Thus, $P_K$ is the unique positive definite solution to the Bellman equation in \eqref{eq:bellman}.
	Meanwhile, 
	since $\EE_{x\sim \rho_K} [ V_K(x)] = 0 $, we have $\alpha_K = - \tr( P_K \Sigma_K)$.
	Hence, we establish \eqref{eq:vk}.
	
	Furthermore, for any state-action pair $(x,u)$, we have 
	\$
	Q_K(x,u) &  = c(x, u) - J(K) + \EE [ V_K(x') \given x, u ]  \notag \\
	& =  c(x, u) - J(K) +  ( A x + Bu)^\top P_K (Ax + Bu)   + \tr ( P_K \Psi  )  - \tr( P_K \Sigma_K )   \notag \\
	& = x^\top Q x + u^\top R u + ( A x + Bu)^\top P_K (Ax + Bu) - \sigma^2 \cdot \tr (R+  P_K BB^\top )-\tr( P_K \Sigma_K ),
	\$ 
	where $x'$ in the first equality is the next state following $(x, u)$, and  the last equality follows from \eqref{eq:cost_K2} and the fact that $\Psi_{\sigma}= \Psi + \sigma^2 \cdot B B^\top$. Thus, we prove \eqref{eq:qk}. 
	
	It remains  to derive the policy gradient $\nabla_K J(K)$.
	By  \eqref{eq:cost_K2}, we have 
	\#\label{eq:pg1}
	\nabla_K J(K) = 2 R K \Sigma_K + \nabla _{K} \tr( Q_0  \cdot  \Sigma_K ) \big \vert_{Q_0 = Q + K^\top R K } ,
	\#
	where the second term denotes that we first take compute the gradient  $\nabla _K \tr [Q _0 \Sigma_K] $ with respect to $K$ and then set $Q_0 = Q + K^\top R K $. Recall that we can write $\Sigma_K = \cT_K(\Psi_{\sigma})$.  The following lemma enables us to compute  the gradient involving $\cT_K$.

	\begin{lemma} 
		Let  $W $ and $\Psi$ be two positive definite  matrices. Then  it holds that 
		\$
		\nabla _{K} \tr\bigl [  W  \cdot    \cT_K (\Psi)\bigr ] = - 2 B^\top \cT_K^\top(W) ( A- BK) \cT_K (\Psi).
		\$
	\end{lemma} 
	\begin{proof}
		To simplify the notation, we define operator $\cF_K$ by 
		\$ 
		\cF_K^\top (\Omega) = ( A- BK)^\top  \Omega ( A- BK)
		\$
		and let $\cF^{\top, t } _K$ be the $t$-th composition of $\cF_K$. Thus, 
		by the definition of $\cT_K^\top$ and $\cF_K^\top$, we have 
		\$
		\cT_K^\top (\Omega) = \sum_{t\geq 0} \cF_K^{\top, t} (\Omega). 
		\$
		Moreover, by \eqref{eq:new_bellman} we have 
		\$
		\tr\bigl [  W  \cdot    \cT_K(\Psi)\bigr ] =  \tr ( W \Psi) +  \tr\bigl [ ( A- BK)^\top  W ( A- BK)   \cdot    \cT_K(\Psi )\bigr ],
		\$
		which implies that 
		\#\label{eq:recur}
		\nabla_K \tr\bigl [  W  \cdot    \cT_K (\Psi)\bigr ]   = -2 B^\top W(A- BK) \cT_K(\Psi) + \nabla_K \tr [ W_1 \cT_K(\Psi )] \Biggiven   _{W_1 = \cF_K(\Omega)  } . 
		\#
		For any $k \geq 1$, by recursively applying \eqref{eq:recur} for $k$ times,  we have 
		\#\label{eq:recur2}
		&	\nabla_K \tr\bigl [  W  \cdot    \cT_K (\Psi)\bigr ]   \notag \\
		& \qquad = -2 B^\top \bigg[ \sum_{t=0}^k  \cF_K^{\top, t} ( W)  \bigg] (A-BK)  \cT_K(\Psi) + \nabla_K \tr [ W_1 \cT_K(\Psi )] \Biggiven  _{W_1 = \cF_K^{(k+1)} (\Omega)  }.
		\#
		Meanwhile, since $\rho(A - BK) < 1$,   we have 
		\$
		\lim_{k\rightarrow \infty} \tr \bigl [ \cF_K^{\top, k} (W) \cT_K(\Psi) \bigr ] \leq  \lim_{k\rightarrow \infty} \| W \| \cdot  \tr [\cT_K(\Psi) ]  \cdot \rho(A- BK)^{2k} = 0.  
		\$
		Thus, by   letting $k$ on the right-hand side of \eqref{eq:recur2}  go to infinity, we obtain
		\$
		\nabla_K \tr\bigl [  W  \cdot    \cT_K(\Psi)\bigr ] = -2 B^\top \bigg[ \sum_{t=0}^\infty   \cF_K^{\top, t} ( W)  \bigg](A-BK)  \cT_K(\Psi) = -2B^\top  \cT_K^{\top} (W) ( A- BK) \cT_K(\Psi).
		\$
		Therefore, we conclude the proof of the lemma.
	\end{proof}
	
	By the above lemma, since $\Sigma_K = \cT_K(\Psi_{\sigma})$, we have 
	\#\label{eq:pg2}
	& \nabla _{K} \tr( Q_0  \cdot  \Sigma_K ) \big \vert_{Q_0 = Q + K^\top R K } = \nabla _{K} \tr\bigl [  Q_0  \cdot    \cT_K (\Psi_{\sigma})\bigr ] \Big \vert_{Q_0 = Q + K^\top R K } \notag \\
	& \qquad = -2 B^\top   \cT_K^\top(Q + K^\top R K) ( A- BK) \cT_K (\Psi_{\sigma}) = -2 B^\top P_K ( A- BK) \Sigma_K,
	\# 
	where we use the fact that $P_K = \cT^\top _K ( Q + K^\top R K )$. 
	Therefore, combining \eqref{eq:pg1} and \eqref{eq:pg2}, we  establish \eqref{eq:grad_cost}, which completes the proof of Proposition \ref{prop:pg}.
\end{proof}

  \subsection{Proof of Lemma \ref{lemma:pe_mat1} } \label{sec:proof_lemma_pe_mat1}
 We present a stronger lemma than Lemma \ref{lemma:pe_mat1}, whose proof automatically validates Lemma \ref{lemma:pe_mat1}.
 \begin{lemma} \label{lemma:pe_mat}
 Suppose $\rho(A- BK)< 1$. Let  $N(0, \tilde \Sigma_K) $ be the stationary distribution of the state-action pair  $(x,u)$ when following policy $\pi_K$.  Then for $\Xi_K$  defined in \eqref{eq:gtd_mat},  we have 
 \#
 \Xi_K & =\bigl ( \tilde\Sigma_K \otimes_s \tilde \Sigma_K \bigr ) - \bigl ( \tilde  \Sigma_K  L  ^\top\bigr)  \otimes_s \bigl ( \tilde   \Sigma_K L  ^\top \bigr) = \bigl ( \tilde\Sigma_K \otimes_s \tilde \Sigma_K \bigr ) \bigl ( I - L^\top \otimes_s L ^\top \bigr ).   \label{eq:xik} 
 \# 
  Moreover,   $\Xi_K$ is a  invertible matrix whose operator norm is bounded by $  2  [   \sigma ^2 +   ( 1+ \| K \|_{\fro}^2  ) \cdot  \| \Sigma_K \|   ] $.  
  There exists a positive number $\kappa_K^*$ such that   the minimum singular value of the matrix in the left-hand side of \eqref{eq:large_le} is lower bounded by a constant $\kappa_K^*>0$, where $\kappa_K^*$ only depends on $\rho(A- BK)$, $\sigma$, and  $\sigma_{\min} (\Psi)$. 
  Furthermore, 
  since $\Xi_K$ is invertible, the linear equation in \eqref{eq:large_le} has  unique solution $\vartheta_K^*$, whose first and second components are $J(K)$ and $\svec(\Theta_K)$, respectively.
 	\end{lemma}
 \begin{proof}	
Throughout the proof of Lemma \ref{lemma:pe_mat}, for any state-action pair $(x, u) 
\in \RR^{d+k}$, 
we denote the next state-action pair following policy $\pi_K$ by $(x' ,u ')$. Then we can write 
\#\label{eq:joint_dynprime}
 x' =  A x + Bu +\epsilon, \qquad u ' = - Kx' + \sigma \cdot \eta = - KA x- K B u - K \epsilon + \sigma \cdot \eta,
\#
where $\epsilon \sim N(0, \Psi)$ and $\eta \in N(0, I_k)$. For notational simplicity, we denote $(x, u)$ and $(x', u')$ by $z $ and $z'$, respectively. Thus, we can write $z ' = L z + \varepsilon$, where we define
\#\label{eq:large_lds}
L = \begin{pmatrix} 
A & B  \\
 - KA & - KB
\end{pmatrix}  =  \begin{pmatrix} 
I _d\\
- K
\end{pmatrix} 
 \begin{pmatrix} 
A & B  
\end{pmatrix} , \qquad \varepsilon   =   \begin{pmatrix} 
 \epsilon\\
- K \epsilon + \sigma \cdot \eta  
\end{pmatrix}. 
\#
Since it holds that $\rho(MN)= \rho(NM)$ for any two matrices $M$ and $N$ \citep[Theorem  1.3.22]{horn2013matrix}, we have $\rho(L ) = \rho( A - BK) < 1$. 
Meanwhile, 
by definition, $\varepsilon \in \RR^{d+k}$ is a centered Gaussian random variable with covariance 
\#\label{eq:def_tildeK}
\begin{pmatrix}
\Psi & - \Psi K^\top  \\
- K \Psi &  K \Psi K^\top + \sigma^2 \cdot  I_k 
\end{pmatrix} , 
\#
which is denoted by $\tilde \Psi_{\sigma}$ for notational simplicity. In addition, for $x\sim \rho_K$ and $u\sim \pi_K(\cdot \given x)$, we denote the joint distribution of $z = (x,u)$ by $\tilde \rho_K$, which is a centered Gaussian distribution in $\RR^{d\times k}$. Since $x \sim N(0 ,\Sigma_K)$ and $u = - K x + \sigma \cdot I_k$, we can write $\tilde \rho_K$ as $N(0, \tilde \Sigma_K)$, where $\tilde\Sigma_K \in \RR^{(d+k)\times (d+k)}$ can be written as 
\# \label{eq:joint_cov}
\tilde \Sigma_K =  \begin{pmatrix}
\Sigma_K &   - \Sigma_K K ^\top \\
- K\Sigma_K & K \Sigma_K K^\top +  \sigma^2 \cdot I_k 
\end{pmatrix}  =  \begin{pmatrix}
	0 &   0 \\
	0 &   \sigma^2 \cdot I_k 
	\end{pmatrix} +
\begin{pmatrix}
	I_d \\
	-K   
\end{pmatrix} \Sigma_K \begin{pmatrix}
I_d \\
-K   
\end{pmatrix} ^\top .
\#
Thus, by triangle inequality   we have 
	\#
	\bigl \| \tilde \Sigma_K \bigl \|_{\fro}  \leq \sigma^2 \cdot k + \| \Sigma_K \|  \cdot ( d + \| K \|_{\fro}^2 )  , \qquad 
	\bigl \| \tilde \Sigma_K \bigl \|   \leq \sigma ^2 +   ( 1+ \| K \|_{\fro}^2  ) \cdot  \| \Sigma_K \| , 
	\label{eq:bound_cov_norm}
	\#
	where in \eqref{eq:bound_cov_norm} we use the fact that $\| A B \|_{\fro} \leq \| A \|_{\fro} \cdot \| B \|$.
	
Furthermore, since $L $ defined in \eqref{eq:large_lds} satisfy $\rho(L) < 1$, 
$\tilde \Sigma_K$ is the unique positive definite solution to the Lyapunov equation 
\#\label{eq:joint_cov_lya}
\tilde \Sigma_K  = L \tilde \Sigma _K  L ^\top +\tilde \Psi_K,
\#
where  $\tilde \Psi_K$ is  defined in    \eqref{eq:def_tildeK}. 
Moreover, the feature mapping can be written as $\phi(x, u) = \phi(z) = \svec(z z^\top)$, which implies that 
\$
    \phi(x, u ) - \phi(x', u') & = \svec\bigl [ z z ^\top - (L z +\varepsilon)(Lz + \varepsilon)^\top \bigr ]  \\
    &= \svec \bigl ( z z ^\top - Lz z ^\top L^\top - L z \varepsilon^\top - \varepsilon z^\top L^\top - \varepsilon\varepsilon^\top\bigr ).
\$
Hence, since $\varepsilon $ is independent of $z $, by the definition of $\Xi_K$ in \eqref{eq:gtd_mat},  we have
\begin{align*}
   \Xi_K  = \EE_{z  \sim \tilde \rho_K}[  \phi(z ) \svec( xx^\top - Lxx^\top L^\top - \tilde \Psi_{\sigma} )^\top ]. 
\end{align*}
 
Now let $M$ and $N$ by any two matrices, by direct computation, we have  
\# 
   & \svec(M)^\top \Xi_K \svec(N)  = \EE_{z \sim \tilde \rho_K } \bigl [\la z z^\top,  M\ra  \cdot  \la z z ^\top - L zz ^\top L^\top - \tilde \Psi_{\sigma},  N\ra  \bigr ] \notag \\
    &\qquad = \EE_{z  \sim \tilde \rho_K }\bigl [ z^\top M z   z ^\top (N - L^\top N L ) z  \bigr ] -     \EE_{z \sim \tilde \rho_K}    [ z ^\top M z    ] \cdot  \la \tilde \Psi_{\sigma} , N \ra \notag  \\
    &\qquad = \EE_{g  \sim N(0, I_{d+k})} \bigl [ g^\top\tilde \Sigma _K ^{1/2} M \tilde \Sigma _K ^{1/2}gg ^\top \tilde \Sigma _K ^{1/2} (N - L^\top N L) \tilde \Sigma _K ^{1/2} g \bigr ] - \big \la  \tilde \Sigma_K,  M \big \ra   \cdot   \big \la \tilde \Psi_{\sigma} , N \big\ra, \label{eq:compute11}
    \#
    where $\tilde \Sigma _K ^{1/2}$ is the  square root of $\tilde \Sigma_K$ defined in \eqref{eq:joint_cov_lya}.
    We utilize the following Lemma to compute the expectation of the product of quadratic forms of Gaussian random variables.
 \begin{lemma}
\label{lem:quadform}
Let 
 $g \sim N(0, I_d)$ be the standard Gaussian random variable in $\RR^d$ and  let $A_1, A_2$ be  two symmetric matrices. Then we have 
 \$
 \EE[ g^\top A_1 g\cdot  g^\top A_2 g  ] = 2 \tr (A_1 A_2) + \tr(A_1) \cdot \tr(A_2).
 \$
\end{lemma}
   \begin{proof}
   See, e.g., \cite{nagar1959bias,magnus1978moments} for a detailed proof.
   \end{proof}
   Applying this lemma to \eqref{eq:compute11}, we have 
   \#
      & \svec(M)^\top \Xi_K \svec(N)  \notag 
      \\&
      \qquad  = 2 \tr \bigl [  \tilde \Sigma _K ^{1/2} M \tilde \Sigma _K ^{1/2}  \cdot  \tilde \Sigma _K ^{1/2} (N - L^\top N L) \tilde \Sigma _K ^{1/2}   \bigr ] \notag   \\
     &  \qquad\qquad   + \tr \bigl ( \tilde \Sigma _K ^{1/2} M \tilde \Sigma _K ^{1/2}  \bigr )  \cdot \tr \bigl [    \tilde \Sigma _K ^{1/2} (N - L^\top N L) \tilde \Sigma _K ^{1/2}   \bigr ] -\la  \tilde \Sigma_K,  M \ra   \cdot  \big \la \tilde \Psi_{\sigma} , N \big \ra \notag \\
     &\qquad  = 2 \big \la  M, \tilde \Sigma _K  (N - L^\top N L)\tilde  \Sigma _K \big\ra     + \big \la M, \tilde \Sigma_K \big \ra  \cdot \bigl [ \bigl \la N - L^\top N L, \tilde \Sigma_K \bigr  \ra - \bigl \la \tilde \Psi_{\sigma} , N \big \ra \bigr ]  . \label{eq:compute12}
   \#
   Note that $\tilde \Sigma_K$ satisfy the Lyapunov equation in  \eqref{eq:joint_cov_lya}, which implies that 
   \$
   \bigl \la N - L^\top N L, \tilde \Sigma_K \bigr  \ra = \bigl \la N, \tilde \Sigma_K  \bigr \ra -  \bigl \la N, L \tilde \Sigma_K L^\top  \bigr \ra = \bigl \la N,  \tilde \Psi_{\sigma} \bigr \ra.
   \$
   Thus, by \eqref{eq:compute12} we have 
   \$
    \svec(M)^\top \Xi_K \svec(N)   & = 2 \big \la  M, \tilde \Sigma _K  (N - L^\top N L)\tilde  \Sigma _K \big\ra =2  \svec (M)^\top \svec\bigl [ \tilde \Sigma _K  (N - L^\top N L)\tilde  \Sigma _K \bigr ]  \\
    & = 2 \svec(M)^\top \bigl (  \tilde \Sigma _K \otimes_{s} \tilde \Sigma _K - \tilde \Sigma _K L^\top \otimes _s  \tilde \Sigma _K L^\top  \bigr ) \svec(N)^\top \notag \\
    & = 2  \svec(M)^\top \bigl [  \bigl ( \tilde \Sigma _K \otimes_{s} \tilde \Sigma _K  \bigr ) ( I- L^\top \otimes L^\top) \bigr ]  \svec(N),
   \$
   where the last equality follows from the fact that 
   \$
   (A \otimes _s B ) ( C \otimes_s D) = 1/2 \cdot (A C \otimes_s BD + AD \otimes _s BC) 
   \$
   holds for any matrices $A$, $B$, $C$, $D$.  
   Thus, we have established \eqref{eq:xik}. Since $\rho(L) = \rho(A- BK) < 1$,  $I- L^\top \otimes L^\top$ is positive definite, which implies that  $\Xi_K$ is invertible.  
   
   Now we consider the  linear equation in 
   \eqref{eq:large_le}. Since $\Xi_K$ is invertible, 
   \#\label{eq:tilde_Xi_K}
    \tilde \Xi_K  = 
   \begin{pmatrix} 
1 & 0\\
\EE_{(x,u)} [ \phi(x,u)]  & \Xi_K 
 \end{pmatrix} =   \begin{pmatrix} 
 1 & 0\\
\svec( \tilde \Sigma_K)  & \Xi_K 
 \end{pmatrix}
    \#
    is also invertible. Thus, 
   \eqref{eq:large_le} has   unique solution $\vartheta_K^*$. 
   Moreover, to bound the smallest singular value of $\tilde \Xi_K$, we note that the inverse of $\tilde \Xi_K$ can be written as 
   \$
   \tilde \Xi_K^{-1} =   \begin{pmatrix} 
1 & 0\\
 -  \Xi_K^{-1}  \svec( \tilde \Sigma_K)   & \Xi_K ^{-1}
 \end{pmatrix},
   \$
   whose operator norm is bounded via 
   \#\label{eq:upper_bound_inv}
   \big \|  \tilde \Xi_K^{-1} \big  \|^2 \leq 1 + \bigl \|  \Xi_K^{-1}  \svec(  \tilde \Sigma_K) \bigr \|_2^2 +  \| \Xi_K ^{-1} \|^2 .
   \#
   By \eqref{eq:xik}, we have 
   \#\label{eq:upper_bound_inv2}
   &  \Xi_K^{-1}   \svec(  \tilde \Sigma_K)   =   ( I - L^\top \otimes_s L ^\top   )^{-1} ( \tilde\Sigma_K \otimes_s \tilde \Sigma_K   )^{-1}   \svec(  \tilde \Sigma_K)  \notag \\
   & \qquad =  ( I - L^\top \otimes_s L ^\top   )^{-1} ( \tilde\Sigma_K^{-1}  \otimes_s \tilde \Sigma_K^{-1}    )  \svec(  \tilde \Sigma_K)   = ( I - L^\top \otimes_s L ^\top   )^{-1} \svec( \tilde \Sigma_K ^{-1} ). 
   \# 
   
   The following lemma characterizes the eigenvalues of symmetric Kronecker matrices.
   
   \begin{lemma} [Lemma 7.2 in \cite{alizadeh1998primal}]\label{lemma:eigkron}
   	Let $A$ and $B$ 
   	be two matrices in $\RR^{m\times m}$ that can be    diagonalized simultaneously.
   	Moreover, let $\lambda_1, \ldots, \lambda_m$ and $\mu_1, \ldots, \mu_m$ be the eigenvalues of $A$ and $B$, respectively. 
   	Then, the eigenvalues of 
   	$A \otimes _s B$ are given by $\{1/2 \cdot  (\lambda_i \mu_j + \lambda_j \mu_i), i,j\in[m] \}$.
   	\end{lemma}
	
	By Lemma  \ref{lemma:eigkron},  the spectral radius of $L^\top \otimes_s L ^\top$ is bounded by $\rho^2 (L)  = \rho^2(A- BK) < 1$. By \eqref{eq:upper_bound_inv2} we have 
	\#\label{eq:upper_bound_inv3}
	 \bigl \| \Xi_K^{-1}   \svec(  \tilde \Sigma_K)  \bigr \|_2  \leq \bigl [ 1- \rho^2 (L)\bigr ]^{-1} \cdot \| \tilde \Sigma _K ^{-1} \|_F \leq \sqrt{d+k} \cdot\bigl [ 1- \rho^2 (L) \bigr ]^{-1} \cdot  \| \tilde \Sigma_K^{-1} \| . 
	\# 
   Besides, by \eqref{eq:xik} we have 
   \#\label{eq:upper_bound_inv4}
    \| \Xi_K ^{-1} \|  \leq  \bigl \| ( I   -  L^\top \otimes_s L ^\top  )^{-1} \bigr \| \cdot \big  \| \tilde\Sigma_K^{-1}  \otimes_s \tilde \Sigma_K^{-1} \big   \| \leq [ 1- \rho^2(L) \bigr ]^{-1} \cdot \big  \|  \tilde \Sigma_K^{-1} \big \|^2 .
   \#
   Notice that $\| \tilde \Sigma_K^{-1}  \| = 1/ \sigma_{\min} ( \tilde \Sigma_K)$.
   Hence, combining \eqref{eq:upper_bound_inv}, \eqref{eq:upper_bound_inv3}, and \eqref{eq:upper_bound_inv4} we conclude that 
   \$
      \big \|  \tilde \Xi_K^{-1} \big  \|^2  \leq 1 + (d+k) \cdot [ 1- \rho(L)^2 \bigr ]^{-2} \cdot [ \sigma_{\min} ( \tilde \Sigma_K)]^{-2} +  [ 1- \rho(L)^2 \bigr ]^{-2 } \cdot [ \sigma_{\min} ( \tilde \Sigma_K)]^{-4},
   \$
   which implies that
    \$
   \sigma_{\min}( \tilde \Xi_K )  \geq \frac{[1- \rho^2 (A- BK)   ]  \cdot [\sigma_{\min} (\tilde \Sigma_K) ]^2 } { \Bigl (  1+  [ 1- \rho^2 (A- BK) ] ^2 \cdot [\sigma_{\min} (\tilde \Sigma_K) ]^4 + (d+k) \cdot  [\sigma_{\min} (\tilde \Sigma_K) ]^2\Bigr )^{1/2}       } > 0.
   \$
  Moreover, to see that $\sigma_{\min} (\tilde \Sigma_K) $ only depends on $\sigma$ and $\sigma_{\min}(\Psi)$, for any $a\in \RR^d$ and $b \in \RR^k$, we have 
  \$
   \begin{pmatrix} 
  a \\
  b
  \end{pmatrix}^\top \tilde  \Sigma_K 
  \begin{pmatrix} 
  a \\
  b
  \end{pmatrix} & = \EE_{(x,u)\sim \tilde \rho_K} [   ( a^\top x + b^\top u) ^2 ]  = \EE_{x\sim \rho_K,\eta \sim N(0, I_k)} \bigl \{ [ (a -K^\top b)  x + \sigma \cdot \eta ]^2 \bigr \}  \\
  &\geq  \sigma^2 \cdot \| b \|_2^2 + \sigma_{\min} ( \Psi) \cdot \| a - K^\top b \|_2^2  \geq ( \sigma^2 - \sigma_{\min} ( \Psi)  \cdot \| K \|^2 ) \cdot \| b \|_2 ^2  +  \sigma_{\min} ( \Psi) \cdot \| a \|_2^2.
  \$
  Thus, suppose $\sigma^2 $ is sufficiently large such that $\sigma^2 - \sigma_{\min} ( \Psi)  \cdot \| K \|^2 > 0$, $\sigma_{\min} (\tilde  \Sigma_K ) $ is lower bounded by $\min\{ \sigma^2 - \sigma_{\min} ( \Psi)  \cdot \| K \|^2, \sigma_{\min} ( \Psi) \}$. Therefore, we can find a constant $\kappa_K^*$ depending only on $\rho(A-KB)$, $\sigma$, and $\sigma_{\min}(\Psi)$ such that $\sigma_{\min}(\tilde \Xi_K) \geq \kappa_K^*$.
     
   Finally, to obtain an upper bound on $\| \Xi_K\|$, by triangle inequality and Lemma  \ref{lemma:eigkron} we have 
   \$   
    \| \Xi_K \| \leq \bigl \| \tilde \Sigma_K \otimes _s \tilde \Sigma_K \bigr \| \cdot \bigl ( 1 + \| L^\top \otimes _s L^\top \| \bigr  )  \leq \bigl \| \tilde \Sigma_K \bigr \| ^2 \cdot  \bigl (1 + \| L \| ^2 \bigr ) \leq 2 \bigl \| \tilde \Sigma_K \bigr \| ^2   ,
   \$
   where we use the fact that $\rho(L) < 1$. 
   Applying \eqref{eq:bound_cov_norm} to the inequality above, we obtain that 
   \$
    \| \Xi_K \| \leq  2 \bigl [   \sigma ^2 +   ( 1+ \| K \|_{\fro}^2  ) \cdot  \| \Sigma_K \|  \bigr ],
   \$
      which concludes the proof.
\end{proof}

\clearpage
\bibliographystyle{ims}
\bibliography{rl_ref}

\end{document}